\documentclass{article}
\usepackage[utf8]{inputenc}
\usepackage{fullpage}

\usepackage{amssymb}
\usepackage{amsmath}
\usepackage{amsthm}
\usepackage{latexsym}
\usepackage{graphicx}
\usepackage{color}
\usepackage{graphics}
\usepackage[round]{natbib}
\usepackage{cancel}
\usepackage{thm-restate}
\usepackage{mathtools}
\usepackage[mathscr]{euscript}
\usepackage{hyperref}
\usepackage{tikz,wrapfig}
\usepackage{cleveref}

\usepackage{enumitem}

\usepackage[algo2e,ruled]{algorithm2e}
\usepackage{algorithm}
\usepackage[noend]{algpseudocode}

\usepackage{MnSymbol}
\DeclareMathAlphabet\mathbb{U}{msb}{m}{n}
\usepackage{xpatch}

\def\Rset{\mathbb{R}}

\let\Pr\undefined

\DeclareMathOperator*{\Pr}{\mathbb{P}}

\DeclareMathOperator*{\E}{\mathbb E}

\DeclareMathOperator{\poly}{poly}

\DeclarePairedDelimiter{\abs}{\lvert}{\rvert}

\newcommand{\Ber}{\mathsf{Ber}}

\newcommand{\cA}{\mathcal{A}}
\newcommand{\cC}{\mathcal{C}}
\newcommand{\cD}{\mathcal{D}}

\newcommand{\cL}{\mathcal{L}}

\newcommand{\cP}{\mathcal{P}}

\newcommand{\cY}{\mathcal{Y}}

\newcommand{\bx}{{\mathbf x}}

\newcommand{\bn}{{\mathbf n}}
\newcommand{\bp}{{\mathbf p}}

\newcommand{\R}{\mathfrak R}

\newcommand{\ForecastHedge}{\textsc{ForecastHedge}}
\newcommand{\ForecastFTPL}{\textsc{ForecastFTPL}}

\newcommand{\wt}{\widetilde}

\newcommand{\eps}{\varepsilon}
\newcommand{\ignore}[1]{}

\newcommand{\wtu}{\widetilde{u}}

\DeclareMathOperator{\sgn}{\mathsf{sgn}}
\DeclareMathOperator{\Reg}{\mathsf{Reg}}
\DeclareMathOperator{\AgentReg}{\mathsf{AgentReg}}
\DeclareMathOperator{\AgentSwapReg}{\mathsf{AgentSwapReg}}
\DeclareMathOperator{\MaxAgentReg}{\mathsf{UCal}}

\DeclareMathOperator{\Cal}{\mathsf{Cal}}
\DeclareMathOperator{\VCal}{\mathsf{VCal}}
\DeclareMathOperator{\VReg}{\mathsf{VReg}}

\newcommand{\sr}{\ell}
\newcommand{\srset}{\cL}
\newcommand{\ramp}{\mathsf{ramp}}

\usepackage{bm}

\newtheorem{theorem}{Theorem}
\newtheorem{definition}{Definition}
\newtheorem{lemma}[theorem]{Lemma}
\newtheorem{corollary}[theorem]{Corollary}
\newtheorem{proposition}[theorem]{Proposition}

\theoremstyle{remark}
\newtheorem{remark}{Remark}
\newtheorem{question}{Question}

\title{U-Calibration: Forecasting for an Unknown Agent}

\author{
Robert Kleinberg\thanks{Cornell University and Google Research. Email: \texttt{rdk@cs.cornell.edu}. This paper was written while the author was a Visiting Faculty Researcher at Google Research.}
\and
Renato Paes Leme\thanks{Google Research. Email: \texttt{renatoppl@google.com}}
\and
Jon Schneider\thanks{Google Research. Email: \texttt{jschnei@google.com}}
\and
Yifeng Teng\thanks{Google Research. Email: \texttt{yifengt@google.com}}
}
\begin{document}

\maketitle

\begin{abstract}
We consider the problem of evaluating forecasts of binary events whose predictions
are consumed by rational agents who take 
an action in response to a prediction, but whose utility is unknown to the forecaster. We show that optimizing forecasts for a single scoring rule (e.g., the Brier score) cannot guarantee low regret for all possible agents. In contrast, forecasts that are well-calibrated guarantee that all agents incur sublinear regret. However, calibration is not a necessary criterion here (it is possible for miscalibrated forecasts to provide good regret guarantees for all possible agents), and calibrated forecasting procedures have provably worse convergence rates than forecasting procedures targeting a single scoring rule.

Motivated by this, we present a new metric for evaluating forecasts that we call
\emph{U-calibration}, equal to the maximal regret of the sequence of forecasts
when evaluated under any bounded scoring rule. We show that sublinear
U-calibration error is a necessary and sufficient condition for all agents to 
achieve sublinear regret guarantees. We additionally demonstrate how to compute
the U-calibration error efficiently and provide an online algorithm that
achieves $O(\sqrt{T})$  U-calibration error (on par with optimal rates for optimizing for a single scoring rule, and bypassing lower bounds for the traditionally calibrated learning procedures). Finally, we discuss generalizations
to the multiclass prediction setting.\footnote{Accepted for presentation at the Conference on Learning Theory (COLT) 2023.}

\end{abstract}

\section{Introduction}

Imagine a weather forecaster who predicts the weather every day. On the morning of the $t$-th day, the forecaster reveals their prediction $p_{t} \in [0, 1]$ for whether it will rain that afternoon (e.g., they might say there is a ``30\% chance of rain that afternoon''). Then, that afternoon, it either rains or it doesn't (in which case we set $x_{t} = 0$ or $x_{t} = 1$ respectively). After many days, we have a large amount of information about both the forecaster's predictions ($p_t$) and the outcomes of the predicted events ($x_t$). Using this information, how should we measure the quality of this forecaster's predictions? Conversely, what sorts of metrics should a good forecaster strive to optimize?

Understanding how to evaluate repeated forecasts is a problem that has been well-studied in many areas, including statistics, computer science, and learning. The most commonly used techniques for performing this evaluation roughly fall into one of two approaches. The first approach is to reward a predictor according to a \textit{proper scoring rule}. A proper scoring rule is a function $\sr: [0, 1] \times \{0, 1\} \rightarrow \Rset$ which takes a prediction $p$ and an outcome $x$, and provides the predictor with a 
``score'' of $\sr(p, x)$. For example, the Brier scoring rule (aka the quadratic scoring rule) penalizes the predictor with a score given by

\begin{equation}\label{eq:brier}
\sr(p, x) = (x - p)^2.
\end{equation}

In order for a scoring rule to be proper, it should incentivize the predictor to predict the true probability (to the best of their knowledge) of the outcome of the corresponding event. Formally, if $x$ is a binary random variable with probability $p$, then $\E_{x}[\sr(p, x)]$ should be  less than $\E_{x}[\sr(p', x)]$ for any $p' \neq p$. It can be checked that the Brier scoring rule \eqref{eq:brier} is proper in this sense, as are many other scoring rules. This motivates measuring the quality of a forecaster by averaging the score of their predictions (e.g. assigning them a score of $\frac{1}{T}\sum_{t=1}^{T}\sr(p_t, x_t)$); if we do this, then to minimize their score it is in the forecaster's interest to predict their true belief about each outcome. 

The second approach is to check how \textit{calibrated} the forecaster is. Intuitively, calibration captures the idea that when the forecaster predicts rain with a probability of 30\%, it should rain about 30\% of the time. In particular, if we aggregate all the forecaster's predictions where the forecaster predicts a specific probability $p$, we should expect roughly a $p$ fraction of the corresponding outcomes to occur. One common way to formalize this is via the following definition of \mbox{($L_1$-)}\textit{calibration error}\footnote{%
    More generally, one can define the $L_p$ calibration error as the 
    $p$-norm of the vector of differences between the probability predicted
    at each time $t$ and of the empirical frequency of positive outcomes
    in all the time steps when the same prediction was made. The calibration
    error formulated in Equation~\eqref{eq:calibration} corresponds to the
    $L_1$ calibration error. The implication that calibration implies low agent regret is only valid for $L_1$-calibration.
} 
as

\begin{equation}\label{eq:calibration}
\Cal = \sum_{p \in [0, 1]} |pn_p - m_p|,
\end{equation}

\noindent
where here $n_p = |\{t \,;\, p_{t} = p\}|$ (the number of times the forecaster predicted $p_t$) and $m_p = |\{t \,;\, p_{t} = p \textrm{ and } x_{t} = 1\}|$ (the number of times the forecaster predicted $p_t$ and the event occurred). Here the term corresponding to each $p$ can be thought of as the error of the forecaster on predictions where they predicted $p$, scaled up by the number of times they predicted $p$ (note that the outer sum is finite since only a finite number of probabilities are ever predicted by the forecaster). 

Both of these approaches to evaluating forecasts suffer from 
drawbacks. To use a proper scoring rule the forecast evaluator
must choose which scoring rule to use. There are infinitely 
many possibilities --- the logarithmic score, Brier score,
and spherical score being three of the most well known --- 
and it is not clear how to assess the benefits and drawbacks
of one proper scoring rule versus another, much less how to 
design a scoring rule optimally for a given application (see e.g. \cite{hartline2020optimization}).

Calibration error gives a canonical way to 
measure forecast accuracy without making any arbitrary choices,
but it lacks a decision-theoretic foundation. In other words, 
the assumption that minimizing calibration error is a desirable goal
for a forecaster or a user of the forecaster's predictions
has no clear justification in terms of those parties' utilities.
Furthermore, algorithms for minimizing calibration error tend to suffer 
from slow convergence. For example, in the binary sequence
prediction problem the forecaster's calibration error in the
worst case is known to be bounded below by $\Omega(T^{0.528})$ (\cite{qiao2021stronger}) 
and above by $O(T^{2/3})$ (\cite{foster1998asymptotic}).
Thus, there is still a very significant gap between the best known
upper and lower bounds, but we already know for certain that the
optimal bound for $L_1$-calibration error is asymptotically greater
than the $O(T^{1/2})$ regret bound that is more typical for other
problems in online learning theory.

In this paper we introduce a new metric 
for forecast evaluation, {\em U-calibration}, 
that overcomes these shortcomings. Informally, 
U-calibration of a forecast sequence 
is defined by evaluating the forecaster's 
regret simultaneously with respect to all 
bounded proper scoring rules and taking the 
maximum regret. Tautologically, U-calibration
implies low scoring rule regret, regardless
of which (bounded) scoring rule is used for
forecast evaluation. U-calibration also has
the following desirable features.
\begin{enumerate}
    \item {\bf Decision-theoretic foundation.}
    Consider an agent facing a repeated decision 
    problem by choosing actions which are best
    responses to the predictions supplied by
    the forecaster. We show in \Cref{sec:agents_to_sr}
    that if the prediction sequence
    has a low (i.e., sublinear) U-calibration score
    with respect to the outcome sequence, then the
    agent will have sublinear regret regardless of
    their utility function. Furthermore, we 
    show that the property of U-calibration is 
    {\em necessary and sufficient} for this 
    ``universal regret minimization'' property. 
    \item {\bf Sublinear U-calibration is achievable.} 
    Achieving sublinear U-calibration requires achieving
    sublinear regret against infinitely many scoring
    rules simultaneously, raising the question of whether
    this property is even attainable in the worst case.
    We show (\Cref{thm:cal_suffices}) that the
    U-calibration score is bounded by a small constant times the 
    $L_1$ calibration error. Hence, any calibrated
    forecasting algorithm can be used to achieve the
    property of U-calibration. 
    \item {\bf Superior rates.} As noted earlier, 
    no forecasting algorithm can achieve $O(T^{1/2})$
    calibration error, even in the case of binary 
    outcomes~\citep{qiao2021stronger}.
    U-calibration does not suffer from this limitation:
    in \Cref{sec:algs} we present a randomized forecasting
    algorithm whose expected U-calibration score is
    $O(T^{1/2})$.
    \item {\bf U-calibration score is easy to compute.}
    Although the informal definition above requires
    maximizing regret over an infinite set of scoring
    rules, we show in \Cref{sec:computing_agent_cal}
    that the U-calibration score of a sequence of forecasts
    and outcomes can be computed in polynomial time. Moreover, we show the U-calibration of a sequence of outcomes is closely related to the regret of the worst ``V-shaped'' scoring rule, differing from this value by at most a factor of 2 (\Cref{thm:vcal_approx}).
\end{enumerate}
To sum up, in any situation in which forecasts are used
to facilitate decision making, a U-calibrated forecast 
sequence is as helpful as an ordinarily calibrated one. 

However, unlike $L_1$-calibration, U-calibration comes
at essentially no cost in terms of regret: there is a forecasting algorithm
which guarantees that agents best-responding to the forecasts
will have $O(\sqrt{T})$ regret, just as if the agents were
directly observing the outcome sequence and 
running optimal full-information learning algorithms 
to make their decisions. 

Finally, it is natural to wonder whether these results extend to predictions over multiple outcomes. In Section \ref{sec:multiclass_u_cal} we define a U-calibration metric in this setting with a similar universal regret minimization property. As in the binary case, we show that the multiclass U-calibration error of a sequence of forecasts is efficiently computable (Section \ref{sec:computing_agent_cal}). Unlike the binary case, however, the structure of worst-case scoring rules seems far more complex in the multi-class setting, and there is no obvious analogue of V-shaped scoring rules. In particular, we show the multiclass U-calibration problem does not reduce to several instances of the binary U-calibration problem: it is possible to be well-calibrated for each outcome individually while having large multi-class U-calibration (Theorem \ref{thm:binary_no_multiclass}). Furthermore, we show that when there are at least $4$ classes, there is no finite-parameter generating basis of all multiclass proper scoring rules the way V-shaped scoring rules form a 1-parameter generating basis for binary scoring rules (Theorem \ref{thm:extremal}).

Nonetheless, we provide a randomized forecasting algorithm which guarantees $O(K\sqrt{T})$ U-calibration error for predictions over $K$ outcomes, with the caveat that this guarantee is slightly weaker than in the binary case -- whereas in the binary case our algorithm minimizes the expected worst-case (over all bounded scoring rules) regret, our multiclass algorithm only guarantees a bound on the worst-case expected regret (Section \ref{sec:multiclass_algs}). Still, this bound is much better than the corresponding $O(T^{K/(K+1)})$ bound on ($K$-multiclass) $L_1$-calibration error attained by existing calibrated forecasting algorithms \citep{foster1997calibrated, blum2008regret}. 

\subsection{Related Work}

The problem of evaluating forecasters and their predictions has a long history spanning many fields. \cite{savage1971elicitation} is one of the first works to introduce proper scoring rules in their generality, but specific scoring rules (e.g. the quadratic scoring rule) appear as far back as \cite{brier1950verification}. Likewise, the idea of employing calibration as a method to evaluate forecasters dates at least as far back as \cite{dawid1982well}. Following from this is a fairly extensive literature \citep{seidenfeld1985calibration, schervish1989general, oakes1985self} discussing which metrics (e.g., scoring rules or calibration error) one should use for evaluating forecasters. Most similar to the perspective we take in this paper is an introductory section of \cite{foster2021forecast} titled ``The Economic Utility of Calibration'', which qualitatively remarks that an agent consuming predictions may benefit from these predictions being calibrated in some sense. \cite{foster2021forecast} do not explore this idea further, instead using this remark to motivate a separate (non-utility-theoretic) procedure called ``forecast hedging''.

The perspective of viewing forecasting as an online learning problem is relatively more recent, largely initiated by \cite{foster1998asymptotic} (who demonstrated an online procedure for producing calibrated forecasts with $O(T^{2/3})$ calibration error; see also \cite{hart2022calibrated}) and \cite{foster1997calibrated}, who showed that calibrated play in games leads to correlated equilibria. Very recently, \cite{qiao2021stronger} proved a lower bound of $\Omega(T^{0.528})$ on the calibration error of any forecaster.

Several variants of calibration have been introduced to deal with the property that calibration error is incredibly sensitive to the precise values of predictions -- perturbing each prediction by a random negligible constant can cause the calibration error to increase by $\Omega(T)$. \cite{kakade2004deterministic} define ``weak calibration'', \cite{foster2018smooth} define ``smooth calibration'', \cite{foster2021forecast} define ``continuous calibration'', and \cite{blasiok2022unifying} define several ``consistent calibration measures'' (many of these notions have additional properties, such as guaranteeing convergence to specific classes of equilibria). Our notion of U-calibration is also robust to slight perturbations but is not captured by any of these existing notions; see Appendix \ref{sec:other-calibration} for a discussion.

The problem of computing U-calibration can be thought of as an optimization problem over scoring rules, a class of problems which has received recent attention in the literature for independent reasons (e.g., it is a useful model for settings such as peer grading). Of most relevance to us is \cite{hartline2020optimization} (where V-shaped scoring rules also play an important role as a solution concept). Other relevant papers include \cite{hartline2022optimal} (a follow-up to \cite{hartline2020optimization} that studies combinatorial settings) and \cite{neyman2021binary} (which optimizes scoring rules that incentivize precision).

Calibration has found a rich collection of applications to problems of group fairness 
through the lens of \emph{multicalibration} \cite{hebert2018multicalibration}. Of this line of work, 
the most related seems to be the very relevant line of work on \emph{omnipredictors} \cite{gopalan2022omnipredictors, gopalan2022loss, gopalan2023characterizing}. An omnipredictor as defined in \cite{gopalan2022omnipredictors} is a predictor (taking as input some features and outputting a probabilistic prediction) that achieves low regret compared to some reference class $\cC$ of hypotheses for \textit{any} loss function in some given class $\cL$ of convex loss functions once the prediction is appropriately transformed. Despite some minor differences in problem set-up (this line of research considers an off-line/contextual model whereas we consider an online / context-free model), this notion of omnipredictor is very similar to a U-calibrated forecaster. These works show that omnipredictors can be constructed from multi-calibrated predictors (in a similar sense as Theorem \ref{thm:cal_suffices}, which shows that calibrated forecasters are U-calibrated). In contrast, we show it is possible to measure U-calibration error and construct online U-calibrated forecasters without directly requiring calibration. It is an interesting question if any of the techniques we discuss in this paper directly extend to the omnipredictor setting.

\section{Model and Preliminaries}
\label{sec:prelims}

\subsection{Scoring rules}
A scoring rule $\sr(p,x)$ is a penalty charged to a forecaster when they predict the probability $p \in [0,1]$ of a binary event $x \in \{0,1\}$. We say it is a  \textit{proper scoring rule} if
$$\E_{x \sim \mathrm{Ber}(p)}[\sr(p, x)] \leq \E_{x \sim \mathrm{Ber}(p)}[\sr(p', x)], \forall p' \neq p$$
where $\mathrm{Ber}(p)$ is a Bernoulli variable of bias $p$. A scoring rule is a \textit{strictly proper scoring rule} if this inequality is strict, i.e., $\E_{x \sim \mathrm{Ber}(p)}[\sr(p, x)] < \E_{x \sim \mathrm{Ber}(p)}[\sr(p', x)]$. Intuitively, a (strictly) proper scoring rule $\sr$ (strictly) incentivizes the forecaster to report the true probability of an event. 

We overload the notation by extending the function linearly to $[0,1]^2$. Let $$\sr(p; q) = \E_{x \sim \Ber(q)}[\sr(p, x)] = (1-q)\sr(p,0) + q\sr(p,1)$$ be the expected penalty from predicting $p$ for a binary event with true probability $q$. Finally, define the univariate form $$\sr(p) = \sr(p; p) = (1-p)\sr(p,0) + p\sr(p,1)$$ as the expected penalty from predicting $p$ for a binary event with true probability $p$. To disambiguate the functions $\sr(p)$ and $\sr(p, x)$, we will refer to the first as the \textit{univariate form} of the scoring rule and the second as the \textit{bivariate form} of the scoring rule. 

The following characterization by \cite{gneiting2007strictly} shows that scoring rules are (essentially) uniquely specified by their univariate form, which may be any concave function (see Appendix \ref{app:omitted} for a proof).

\begin{lemma}\label{lem:sr_conv}
Given any scoring rule $\sr$, the univariate form $\sr(p)$ is a concave function over the interval $[0, 1]$. Moreover, given any concave function $f:[0, 1]\rightarrow \Rset$, there exists a  scoring rule $\sr$ such that $\sr(p) = f(p)$ for $p \in [0, 1]$. Finally, if $\ell(p)$ is differentiable, then we can recover the bivariate form $\ell(p, x)$ via the equations

\begin{equation}\label{eq:scoring_rule_via_derivative}
\sr(p, 0) = \sr(p) - p\sr'(p) \qquad \qquad 
\sr(p, 1) = \sr(p) + (1-p)\sr'(p).
\end{equation}
\end{lemma}

Unless otherwise specified, we will only concern ourselves with \textit{bounded} scoring rules whose range lies in the interval $[-1, 1]$. This will imply a bound on the derivative of the univariate form:

\begin{corollary}\label{cor:derivative_bound}
For any scoring rule with range $\ell(p,x) \in [-1,1]$ the derivative of the univariate form is bounded:  $\ell'(p) \leq 2$.
\end{corollary}

\begin{proof}
    By equation \eqref{eq:scoring_rule_via_derivative} we have $\sr'(p) = \sr(p,1) - \sr(p,0) \in [-2,2]$ since $\sr(p,x) \in [-1,1]$.
\end{proof}

There are many different scoring rules that are commonly used in practice (e.g., Brier, logarithmic, spherical, etc.). The only scoring rule we will mention by name is the Brier scoring rule, defined by $\sr_{sq}(p, x) = (x-p)^2$, which has the univariate form $\sr_{sq}(p) = p(1-p)$.

\subsection{Forecasters and Agents}

We consider the following repeated game (which takes place over $T$ rounds) between three players: an Adversary, a Forecaster, and an Agent. The Adversary begins the game\footnote{For simplicity, we work in the oblivious model where the adversary must fix the sequence of outcomes at the very beginning of the game (this is the strongest model for our negative results). } by selecting for each $1 \leq t \leq T$, the outcome of a binary event $x_{t} \in \{0, 1\}$. 

The Forecaster's goal is to predict the outcomes of the events $x_{t}$ accurately. At the beginning of round $t$, the Forecaster outputs a prediction $p_{t} \in [0, 1]$ for $x_t$ as a (randomized) function of the previous predictions $p_1, \hdots, p_{t-1}$ and outcomes $x_1, \hdots, x_{t-1}$. We will discuss shortly several options for measuring the quality of the Forecaster's predictions.

Finally, the Agent must use the prediction $p_t$ provided by the Forecaster to choose an action $a_t$ (in some finite set of possible actions $\cA$) to take on round $t$. The utility of this action for the Agent depends on both the choice of action and the outcome of the event. Formally, we assume the existence of a bounded utility function $u: \cA \times \{0, 1\} \rightarrow [-1, 1]$ such that the agent receives utility $u(a, x)$ for playing action $a$ when outcome $x$ occurs. The agent trusts the Forecaster and chooses the action $a_{t}$ which maximizes $\E_{x \sim \Ber(p_t)}[u(a_t, x)]$ (i.e., the optimal action under the assumption that the outcome $x_t$ truly has probability $p_t$ of occurring). The Agent would like to maximize their total utility $\sum_{t} u(a_t, x_t)$. In practice, since the Agent's actions directly follow from the Forecaster's predictions, this will be one way we evaluate the Forecaster's predictions. 

We define the base rate frequency for the event occurring as:
$$\beta = \frac{1}{T}\sum_{t} x_t.$$ We will consider several methods for evaluating the Forecaster, each of which compares the Forecaster to the hypothetical \textit{base rate forecaster}, who predicts $p_{t} = \beta$ every round\footnote{In this sense each of our metrics is a form of \textit{regret}, as they compare our online Forecaster to the best fixed-prediction forecaster in hindsight.}. These are:

\begin{enumerate}[leftmargin=0.4cm]
    \item \textbf{Brier score / scoring rule regret}. One reasonable objective for the Forecaster is to minimize their total Brier score. We define the \textit{regret} of the Forecaster to be the difference between their total Brier score and the Brier score of the base rate forecaster. That is, for a sequence of $T$ binary events $\bx$ and corresponding predictions $\bp$ by the Forecaster, we define

    \begin{equation}
    \Reg(\bp, \bx) = \sum_{t=1}^{T} \sr_{sq}(p_t, x_t) - \sum_{t=1}^{T} \sr_{sq}(\beta, x_t).
    \end{equation}

    \noindent
    We will omit the parameters $\bp$ and $\bx$ when they are clear from the context. We say that the Forecaster has \textit{low regret} if (in expectation over the randomness in the Forecaster's algorithm) $\Reg = o(T)$, high regret if $\Reg \geq \Omega(T)$,  and \textit{negative regret} if $\Reg \leq -\Omega(T)$. A low regret Forecaster is at least as good (up to sublinear in $T$ terms) as the base rate forecaster and a negative regret Forecaster has a significant (linear in $T$) advantage over the base rate forecaster (when evaluated via Brier scores).
    
    Of course, we can extend this definition to an arbitrary fixed scoring rule $\sr$ and similarly write

    \begin{equation}\label{eq:sr_reg}
    \Reg_{\sr}(\bp, \bx) = \sum_{t=1}^{T} \sr(p_t, x_t) - \sum_{t=1}^{T} \sr(\beta, x_t).
    \end{equation}

    \noindent
    Likewise, we say that a Forecaster has \textit{low regret for (scoring rule) $\sr$} if $\Reg_{\sr} = o(T)$, \textit{high regret for (scoring rule) $\sr$} if $\Reg_{\sr} \geq \Omega(T)$, and \textit{negative regret for (scoring rule) $\sr$} if $\Reg_{\sr} \leq -\Omega(T)$. 
    
    \item \textbf{Calibration}. As in the introduction, we define the calibration of the Forecaster via

    \begin{equation}\label{eq:calibration_redef}
    \Cal(\bp, \bx) = \sum_{p \in [0, 1]} |pn_p - m_p|,
    \end{equation}

    \noindent
    where $n_p = |\{t \,;\, p_{t} = p\}|$ (the number of times the forecaster predicted $p_t$) and $m_p = |\{t \,;\, p_{t} = p \textrm{ and } x_{t} = 1\}|$ (the number of times the forecaster predicted $p_t$ and the event occurred). We say a Forecaster is \textit{well-calibrated} if $\Cal = o(T)$, and \textit{poorly calibrated} if $\Cal \geq \Omega(T)$. Note that the base rate forecaster has zero calibration error, so again this can be thought of as the difference between the Forecaster's performance and the base rate forecaster's performance. 

    \item \textbf{Agent utility}. Finally, we compare the Agent's utility under following the Forecaster's predictions with their counterfactual utility from following the base rate forecaster's predictions. In particular, we define the Agent's regret (for an agent with utility function $u$) as

    \begin{equation}\label{eq:agent_reg}
    \AgentReg_u(\bp, \bx) = \sum_{t=1}^{T} u(a_{\beta}, x_t) - \sum_{t=1}^{T} u(a_t, x_t),
    \end{equation}

    \noindent
    where $a_t = \arg\max_{a_t \in \cA} \E_{x \sim \Ber(p_t)}[u(a_t, x)]$ and $a_{\beta} = \arg\max_{a_{\beta} \in \cA} \E_{x \sim \Ber(\beta)}[u(a_{\beta}, x)]$. As with the scoring rules, we say that the Forecaster has \textit{low regret for the agent} if $\AgentReg_u = o(T)$, \textit{high regret for the agent} if $\AgentReg_u \geq \Omega(T)$, and \textit{negative regret for the agent} if $\AgentReg_u \leq -\Omega(T)$. In fact, as we will see in Section \ref{sec:agents_to_sr}, $\AgentReg_u$ is a special case of scoring rule regret for a properly defined scoring rule $\sr$. 
\end{enumerate}

It follows from known results in the online learning and optimization literature that the above low regret guarantees are all achievable -- see e.g. \citep{foster1997calibrated, arora2012multiplicative}.

\subsection{Agents as scoring rules}\label{sec:agents_to_sr}

We now show that optimizing the utility of the Agent corresponds to minimizing a specific scoring rule, thus connecting the benchmarks $\AgentReg_u$ and $\Reg_{\sr}$. Define $\wt{u}(p, x) = u(a(p), x)$ where 

$$a(p) = \arg\max_{a \in \cA} \E_{x \sim \Ber(p)}[u(a, x)]$$

\noindent
is the optimal action for the agent if the true probability of the event is $p$. In other words, $\wt{u}(p, x)$ is the utility the agent receives when receiving a prediction $p$ for an event with actual outcome $x$. We have the following lemma.

\begin{lemma}\label{lem:agent_to_sr}
Let $\sr(p, x) = -\wt{u}(p, x)$. Then $\sr$ is a proper scoring rule and $\AgentReg_u = \Reg_{\sr}$. Moreover, if $\sr$ is a proper scoring rule such that $\sr(p)$ is piecewise linear, there exists a utility function $u$ such that $\wt{u}(p, x) = -\sr(p, x)$. 
\end{lemma}
\begin{proof}
To show $\sr$ is a proper scoring rule, we must show that $\E_{x \sim \mathrm{Ber}(p)}[\sr(p, x)] \leq \E_{x \sim \mathrm{Ber}(p)}[\sr(p', x)]$ for any $p' \neq p$. Equivalently, we must show that $\E_{x \sim \mathrm{Ber}(p)}[u(a(p), x)] \geq \E_{x \sim \mathrm{Ber}(p)}[u(a(p'), x)]$ for any $p' \neq p$. But since $a(p) = \arg\max_{a \in \cA}\E_{x \sim \Ber(p)}[u(a, x)]$, this inequality immediately follows.

Furthermore, note that in the definition of $\AgentReg_u$ in \eqref{eq:agent_reg}, $u(a_{\beta}, x_{t}) = -\sr(\beta, x_t)$ and $u(a_t, x_t) = -\sr(p_t, x_t)$. Making these substitutions, it is clear that $\AgentReg_u = \Reg_{\sr}$. 

In the other direction, if $\sr(p)$ is piecewise linear and concave (since $\sr$ is a proper scoring rule), then we can write $\sr(p) = \min_{i \in [K]}(r_ip + s_i)$ for some collection of $K$ linear functions $r_ip + s_i$. Consider the agent with $\cA = [K]$ and $u(a, x) = -(r_ix + s_i)$. Then $\wt{u}(p, x) = \max_{a \in \cA} -(r_ap + s_a) = - \min_{i \in [K]}(r_ip + s_i) = -\sr(p, x)$.
\end{proof}

In the remainder of this paper, we will take the perspective of a Forecaster who does not know the Agent's utility function $u$, yet nevertheless wants to guarantee low regret for the agent. That is, the Forecaster would like an arbitrary Agent to be (approximately) at least as well off by trusting the Forecaster's predictions than by simply assuming events occur at the base rate. Equivalently (by Lemma \ref{lem:agent_to_sr}), the Forecaster would like to have low regret with respect to all (bounded) scoring rules $\sr$. 

Two questions immediately arise: 1. Is it sufficient for the Forecaster to have low regret with respect to some specific scoring rule (e.g. the Brier scoring rule)? and 2. Is it sufficient for the Forecaster to be well calibrated? We address these in the next section.

\section{Calibration versus scoring rules}
\label{sec:calibration_vs_scoring}

\subsection{Low Brier scores can lead to high agent regret}

We begin by addressing the question of whether it is sufficient for the Agent to follow a Forecaster with low Brier score (specifically, low Brier score compared to the base rate forecaster). We show that the answer is \textit{no}; there are cases where an Agent can lose $\Omega(T)$ utility by following some specific Forecaster over the base rate forecaster, even if this Forecaster has an equal or better Brier score than the base rate forecaster.

\begin{theorem}\label{thm:low_brier}
There exists a sequence of $T$ binary events $\bx$, $T$ forecasts $\bp$, and a utility function $u$ where $\Reg(\bp, \bx) = -\Omega(T)$ but $\AgentReg_u(\bp, \bx) = \Omega(T)$. 
\end{theorem}

\begin{proof}
Consider the sequence of $T$ binary events where for the first half of the $T$ events $x_t = 1$ and for the second half of the $T$ events $x_t = 0$. In both halves, the Forecaster will correctly predict $p_t = x_t$ for $80\%$ of the events, and incorrectly predict $p_t = 1-x_t$ for the remaining $20\%$ of the events. Note that the total Brier score of these forecasts is equal to $\sum_t \sr_{sq}(p_t, x_t) = 0.2T$ (the Forecaster incurs a penalty of $1$ every time they predict incorrectly), which is less than the Brier score of the base rate Forecaster (who always predicts $1/2$ and incurs a penalty of $1/4$ every round. It's therefore the case that $\Reg(\bp, \bx) = -0.05T = -\Omega(T)$.

To define $u$, we will offer the Agent two actions (which we can think of as wagers at 9-to-1 odds); either they can bet that $x_t = 0$, whereupon they receive a reward of $0.1$ if they are correct and a penalty of $0.9$ if they are incorrect, or bet that $x_t = 1$, whereupon they receive a reward of $0.9$ if they are correct and a penalty of $0.1$ if they are incorrect. Formally, we can write $u(a, x) = (-1)^{a}(0.1(1-x) - 0.9x) = (-1)^{a}(0.1-x)$, where the Agent's action $a$ is their prediction for $x$. See Figure \ref{fig:brier_example}. Note that the Agent will predict $a_t = 1$ in exactly the rounds where the forecast $p_t \geq 0.1$. 

An Agent following the base rate forecaster will always predict $a_t = 1$ (since $1/2 \geq 0.1$), and receive a total utility of $(0.9)(T/2) - (0.1)(T/2) = 0.4T$. On the other hand, an Agent following the forecasts described above will predict $a_t = p_t$ and receive a total utility of $(0.9)(0.4T) - (0.9)(0.1T) + (0.1)(0.4T) - (0.1)(0.1T) = 0.3T$. It follows that in this example, $\AgentReg_u = 0.1T = \Omega(T)$.
\end{proof}

 \begin{figure}[h]
\centering
\scalebox{0.75}{
\begin{tikzpicture}[scale=3]

\draw (0,0) -- (1,0);
\draw (0,-1) -- (0,1);
\draw[dashed] (0,0.1) -- (1,-0.9);
\draw[dashed] (0,-0.1) -- (1,0.9);
\draw[line width=1.5, blue] (0,.1) -- (.1,0) -- (1,.9);
\node at (1.2,.8) {$u(0,p)$};
\node at (1.2,-.8) {$u(1,p)$};
\node at (1.2,0.) {$p$};
\end{tikzpicture}}
\caption{Utility function in the proof of Theorem \ref{thm:low_brier}}
\label{fig:brier_example}
 \end{figure}
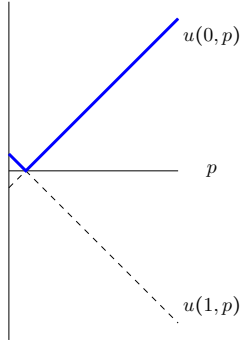

In fact, the property of Theorem \ref{thm:low_brier} extends to any scoring rule, not just the Brier scoring rule. That is, there is no single scoring rule $\sr$ where a Forecaster's forecasts outperforming the base rate forecasts on scores from $\sr$ implies that an arbitrary Agent should follow these forecasts over the base rate forecasts. 

\begin{theorem}\label{thm:any_sr_reg}
Given any bounded proper scoring rule $\sr$, there exists another proper scoring rule $\tilde\sr$ such that for any sufficiently large $T$, there exists a sequence of $T$ forecasts $\bp$ and binary events $\bx$ such that $\Reg_{\sr}(\bp, \bx) = o(T)$ but $\Reg_{\tilde\sr}(\bp, \bx) = \Omega(T)$.
\end{theorem}

\begin{proof}
The case where $\ell(p)$ is linear is trivial since any Forecaster has $\Reg_\ell(\bp, \bx) = 0$.
If $\ell(p)$ is non-linear, then $\ell'(0) > \ell'(1)$. Let $x_t = 0$ for $T/2$ rounds and $x_t = 1$ for $T/2$ rounds in any given order. The benchmark is $T \ell(1/2)$. Consider a Forecaster that always predicts $p_t \in \{0,1\}$, predicting incorrectly $p_t = 1-x_t$ for a fraction $f \in [0, 1]$ of the rounds and correctly otherwise in a balanced way such that the number of correct predictions of $0$s and $1$s is the same. The score of this Forecaster is:
$$(1-f) T \left( \frac{\ell(0) + \ell(1)}{2} \right) + f T \left( \frac{\ell(0) + \ell'(0) + \ell(1) - \ell'(1)}{2} \right) =  T \left( \frac{\ell(0) + \ell(1)}{2} \right) + f T \left( \frac{\ell'(0) - \ell'(1)}{2} \right)$$
Now, choose $f \in (0,1)$ such that the above expression is equal to $T \ell(1/2)$. This is always possible since:
$$  \frac{\ell(0) + \ell(1)}{2} < \ell\left(\frac{1}{2}\right) <  \frac{\ell(0) + \ell'(0) + \ell(1) - \ell'(1)}{2} $$
by concavity and the fact that $\ell'(0) > \ell'(1)$.  Since the performance of this Forecast matches the performance of the base rate forecast, $\Reg_{\ell}(\bp, \bx) = 0$. Now, construct a scoring rule $\tilde{\ell}$ which leads to the same algorithm performance but has an improved benchmark. For example:
$$\tilde\ell(p) = \min(\ell(p), \ell(0) + p (\ell(1) - \ell(0)) + \epsilon)$$
for some very small $\epsilon$. The performance of the algorithm is still the same since $\tilde\ell(p,x) = \ell(p,x)$ is still the same for $p \in \{0,1\}$ but the base rate forecaster has now performance $T \tilde\ell(1/2) = T[\frac{1}{2}(\ell(0) + \ell(1)) + \epsilon]$. Hence 
$$\Reg_{\tilde\ell}(\bp, \bx) = T\left( \ell\left(\frac{1}{2}\right) - \frac{\ell(0) + \ell(1)}{2} - \epsilon \right)$$
which is linear for sufficiently small values of $\epsilon$. 
\end{proof}

\subsection{Calibration leads to sublinear agent regret}

Despite this, it is the case that agents cannot go wrong by trusting forecasters that are well-calibrated -- in particular, we show that the regret of any agent is bounded above by a small multiple of the calibration error of the Forecaster. Intuitively, this follows from the fact that in well-calibrated forecasts, if an Agent often sees a prediction of exactly $p$, the empirical probability of the event will be very close to $p$.

\begin{theorem}\label{thm:cal_suffices}
For any sequence of $T$ binary events $\bx$, predictions $\bp$, and bounded agent (with utility $u$), we have that $\AgentReg_u(\bp, \bx) \leq 4\Cal(\bp, \bx)$. In particular, if $\bp$ and $\bx$ satisfy $\Cal(\bp, \bx) = o(T)$, then for any bounded agent, $\AgentReg_u(\bp, \bx) = o(T)$. 
\end{theorem}

\begin{proof}
Let $\sr(\bp, \bx)$ be the scoring rule corresponding to this agent, so (by Lemma \ref{lem:agent_to_sr}) we wish to show that $\Reg_{\ell}(\bp, \bx) = o(T)$. We will need the following fact about bounded scoring rules $\ell$. For any $p, \hat{p} \in [0, 1]$, the following inequality holds:

\begin{equation}\label{eq:sr_conv3}
\sr(\hat{p}) \leq \sr(p;\hat{p}) \leq \sr(\hat{p}) + 4\left|p - \hat{p}\right|
\end{equation}

The first inequality in \eqref{eq:sr_conv3} follows from the fact that the scoring rule is proper, i.e., for a fixed $\hat{p}$ $\sr(p;\hat{p})$ is minimized when $p = \hat{p}$. To prove the second inequality, first write $\sr(p; \hat{p})$ in the form $\sr(p) + (\hat{p} - p)\sr'(p)$ (as in the proof of Lemma \ref{lem:sr_conv}), and then apply the fact that $|\sr'(p)| \leq 2$ (Corollary \ref{cor:derivative_bound}) to show that $\sr(p; \hat{p}) \leq \sr(p) + 2|\hat{p} - p|$. Finally, from concavity of $\sr$ (and Corollary \ref{cor:derivative_bound} again), we have that $\sr(p) \leq \sr(\hat{p}) + (p - \hat{p})\sr'(\hat{p}) \leq \sr(\hat{p}) + 2|p - \hat{p}|$. Combining these two inequalities we obtain \eqref{eq:sr_conv3}.

Now, note that

\begin{eqnarray*}
    \Reg_{\sr}(\bp, \bx) &=& \sum_{t=1}^{T} \sr(p_t, x_t) - \sum_{t=1}^{T} \sr(\beta, x_t) \\
    &=& \sum_{p \in [0, 1]} \sum_{t; p_t = p}\left(\sr(p, x_t) - \sr(\beta, x_t)\right) \\
    &=& \sum_{p \in [0, 1]} \left((n_p - m_p)(\sr(p,0) - \sr(\beta,0)) + m_p(\sr(p,1) - \sr(\beta,1))\right) \\
    &=& \sum_{p \in [0, 1]} n_p\left(\sr\left(p;\frac{m_p}{n_p}\right) - \sr\left(\beta; \frac{m_p}{n_p}\right)\right) \\
    &\leq& 4 \sum_{p \in [0, 1]} n_p \left| p - \frac{m_p}{n_p}\right| = 4 \Cal(\bp, \bx).
\end{eqnarray*}

\noindent
Here the last inequality follows from applying  \eqref{eq:sr_conv3}.

\end{proof}

\section{U-calibration}
\label{sec:u_calibration}

In the previous section, we have shown that if our goal is to simultaneously achieve sublinear agent regret for all possible agents (equivalently, achieve sublinear scoring rule regret for all possible scoring rules), it suffices that we employ a calibrated forecasting procedure. This begs the question: is a calibrated forecasting procedure \textit{necessary} for obtaining sublinear agent regret for all possible agents?

In particular, can we obtain regret better than what is possible under a calibrated forecast? Using an algorithm (such as \cite{foster1998asymptotic} or \cite{blum2007external}) we can obtain $O(T^{2/3})$ calibration error and hence $O(T^{2/3})$ regret simultaneously for all possible agents. At the same time, \cite{qiao2021stronger} recently showed a lower bound of $\Omega(T^{0.528})$ for calibrated forecasts.

In this section, we show that calibration is not necessary to obtain low regret for all possible agents. In fact, it is possible to bypass the lower bound of \cite{qiao2021stronger}  and obtain an algorithm with regret $O(T^{1/2})$ for all possible agents, asymptotically matching the optimal guarantee obtainable if we were to know the utility function in advance.

To get some intuition for why calibration may not be necessary, note that the calibration error function $\Cal(\bp, \bx)$ is extremely sensitive to small perturbations in the predictions $\bp$, whereas (for any bounded agent) $\AgentReg(\bp, \bx)$ is not. We formally show this in the following lemma.

\begin{lemma}\label{lem:cal_not_necessary}
There exists a sequence of $T$ predictions $\bp$ and binary events $\bx$ where $\Cal(\bp, \bx) = \Omega(T)$ but for any choice of bounded scoring rule $\sr$, $\Reg_{\sr}(\bp, \bx) = o(T)$.
\end{lemma}
\begin{proof}
Begin by letting $\bx$ be a sequence of binary events with $T/2$ zeros and $T/2$ ones (in any order), and let $\bp$ be the constant base rate prediction of $p_t = 1/2$. This prediction has calibration error $\Cal(\bp, \bx) = 0$, so by Theorem \ref{thm:cal_suffices}, $\Reg_{\ell}(\bp, \bx) = o(T)$ for any bounded scoring rule $\ell$. We now define a new (perturbed) sequence of predictions $\bp'$ as follows: for each $t$ where $x_t = 0$, set $p'_t = p_t - z_t$, and for each $t$ where $x_t = 1$, set $p'_t = p_t + z_t$, where the $z_t$ are all distinct real numbers in the interval $[0, 0.001]$. Since each $p'_t$ moves $p_t$ closer to $x_t$, for each scoring rule $\sr$, $\Reg_{\sr}(\bp', \bx) \leq \Reg_{\sr}(\bp, \bx) = o(T)$, since by equation \eqref{eq:scoring_rule_via_derivative} the functions $\ell(p,0)$ and $\ell(p,1)$ are monotone\footnote{To see that $\ell(p,0)$ is monotone, observe that for $p \leq q$ we have $\sr(p,0) = \sr(p) - p \sr'(p) \leq \sr(q) - q \sr'(p)$ by concavity of $\sr$. Also by concavity $\sr'(p) \geq \sr'(q)$ so: $\sr(p,0) \leq \sr(q) - q \sr'(q) = \sr(q,0)$. The argument for $\sr(p,1)$ is similar.
}. On the other hand, since the $z_t$ are all distinct, each probability is predicted exactly once and $\Cal(\bp', \bx) \geq 0.499T = \Omega(T)$.
\end{proof}

Given the result of Lemma \ref{lem:cal_not_necessary}, it is natural to ask whether there is some version of calibration which captures exactly this notion of producing good forecasts simultaneously for all possible agents. The goal of the remainder of this section is to define such a notion (which we call \textit{U-calibration}) and establish some of its basic properties -- how to compute this quantity, how it compares to other versions of calibration, how to design algorithms to minimize this quantity, etc.

\subsection{Defining U-calibration and V-calibration}

If our goal is to simultaneously minimize the regret with respect to every single scoring rule, it makes sense to measure the regret of the worst scoring rule. Define the set $\srset$ to be the set of all bounded proper scoring rules $\sr$:
$$\srset = \{ \ell:[0,1] \times \{0,1\} \rightarrow [-1,1]; \ell \text{ is a proper scoring rule}\} $$
We define the \textit{U-calibration error} $\MaxAgentReg$ to be the maximum regret of any bounded agent, or equivalently,

\begin{equation}\label{eq:max_agent_reg}
\MaxAgentReg(\bp, \bx) = \sup_{\sr \in \srset} \Reg_{\sr}(\bp, \bx). 
\end{equation}

The main downside of this definition is that this requires an optimization over all scoring rules $\sr \in \srset$, which is not a priori obvious how to perform\footnote{Although it is possible to perform this optimization efficiently -- see Theorem \ref{thm:agent_cal_alg} in \Cref{sec:multiclass_u_cal}, where we describe how to do this for the more general case of $K$ outcomes.}. We will introduce a relaxation of U-calibration that we call \emph{V-calibration}, which will be defined similarly to \eqref{eq:max_agent_reg}, except that we will take the maximum over a much smaller (but still representative) collection of scoring rules we call V-shaped scoring rules. The \textit{V-shaped scoring rule $\sr_v$ centered at $v \in [0, 1]$} is defined to be the scoring rule with univariate form $\sr_{v}(p) = -|p - v|$. We then define the \textit{V-calibration error} of a sequence of predictions to be

\begin{equation}\label{eq:vcal}
\VCal(\bp, \bx) = \sup_{v \in [0, 1]} \Reg_{\sr_{v}}(\bp, \bx). 
\end{equation}

 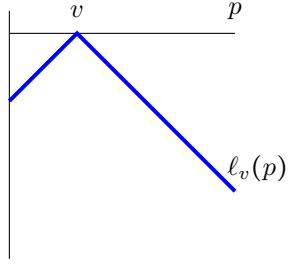
\begin{figure}[h]
\centering
\begin{tikzpicture}[scale=3]

\draw (0,0) -- (1,0);
\draw (0,-1) -- (0,.1);
\draw[line width=1.5, blue] (0,-.3) -- (.3,0) -- (1,-.7);
\node at (.3,.1) {$v$};
\node at (1.1,-.6) {$\ell_v(p)$};
\node at (1,.1) {$p$};
\end{tikzpicture}
\caption{Example of a $V$-shaped scoring rule}
\label{fig:ell_v}
 \end{figure}

One reason to focus on V-shaped scoring rules is that (as we shall soon show) they form a natural and efficient basis for the set of all bounded scoring rules. One consequence of this is that our definition of V-calibration error is a constant factor approximation to the agent calibration error. 

\begin{theorem}\label{thm:vcal_approx}
For any sequence of $T$ predictions $\bp$ and binary events $\bx$, we have that

$$\frac{1}{2} \cdot \MaxAgentReg(\bp, \bx) \leq \VCal(\bp, \bx) \leq \MaxAgentReg(\bp, \bx).$$
\end{theorem}

\begin{proof}
Note that since each V-shaped scoring rule belongs to $\srset$, the right inequality immediately follows. We therefore only need to prove the left side of the above inequality. At a high level, we will show that it is in fact possible to decompose any bounded scoring rule $\sr$ into a positive linear combination of V-shaped scoring rules; i.e., up to an additive linear term (which does not affect regret), one can write $\sr(p) = \int_{0}^{1}\mu(v)\sr_{v}(p)dv$ for some measure $\mu$ over $[0, 1]$ with weight at most $2$. The approximation guarantee then follows from the fact that $\Reg_{\sr}$ is a linear functional in $\sr$. 

Fix a choice of $\bp$ and $\bx$. We begin by rewriting $\Reg_{\sr}(\bp, \bx)$ for a generic $\sr \in \srset$ in terms of the univariate form of the scoring rule (by applying the identity $\sr(p, x) = \sr(p) + (x-p)\sr'(p)$). We have

\begin{eqnarray*}
    \Reg_{\sr}(\bp, \bx) &=& \sum_{t=1}^{T} \sr(p_t, x_t) - \sum_{t=1}^{T} \sr(\beta, x_t) \\
    &=& \sum_{t=1}^{T} \left(\sr(p_t) + (x_t - p_t)\sr'(p_t) - (\sr(\beta) + (x_t - \beta)\sr'(\beta))\right) \\
    &=& \left(\sum_{t=1}^{T} \sr(p_t) + (x_t - p_t)\sr'(p_t)\right) - T\sr(\beta).
\end{eqnarray*}

We now make the following observations about $\Reg_{\sr}(\bp, \bx)$:

\begin{itemize}
    \item First (and most importantly), $\Reg_{\sr}(\bp, \bx)$ is linear in (the univariate form of) $\sr$. In particular, for any $\sr$ and $\tilde\sr$ in $\srset$, we have that $\Reg_{\sr + \tilde\sr}(\bp, \bx) = \Reg_{\sr}(\bp, \bx) + \Reg_{\tilde\sr}(\bp, \bx)$, and $\Reg_{\lambda\sr}(\bp, \bx) = \lambda \Reg_{\sr}(\bp, \bx)$ for any $\lambda \geq 0$. 
    \item Secondly, $\Reg_{\sr}(\bp, \bx)$ is invariant upon the addition of constant or linear functions to (the univariate form) of $\sr$. Specifically, for any constants $C_0, C_1$, if we construct the scoring rule $\tilde\sr(p) = \sr(p) + C_1p + C_0$, then $\Reg_{\tilde\sr}(p) = \Reg_{\sr}(p)$.
\end{itemize}

Now, since the collection of scoring rules with piece-wise linear univariate forms is dense in $\srset$, assume without loss of generality that $\sr(p)$ is a piece-wise linear function in $p$. If $\sr(p)$ has $k$ breakpoints at values $v_1, v_2, \dots, v_k$, we claim we can write

\begin{equation}\label{eq:sr_decomp}
\sr(p) = C_1p + C_0 + \sum_{i=1}^{k}\lambda_{i}\sr_{v_i}(p),
\end{equation}

\noindent
for some constants $C_0, C_1 \in \Rset$ and nonnegative reals $\lambda_i$ such that $\sum_{i} \lambda_i \leq 2$. To see this, first recall that any piece-wise linear function $\sr(p)$ with breakpoints at $v_i$ can be written in the form

\begin{equation}\label{eq:sr_deriv1}
\sr(p) = \sr(0) + p\sr'(0) + \sum_{i=1}^{k}(\sr'_{+}(v_i) - \sr'_{-}(v_i))\cdot\ramp(p - v_i),
\end{equation}

\noindent
where we define $\sr'_{+}(v) = \lim_{p \rightarrow v^{+}} \sr'(p)$, $\sr'_{-}(v) = \lim_{p \rightarrow v^{-}} \sr'(p)$, and $\ramp(p) = \max(p, 0)$ to be the piece-wise linear ``ramp'' function. Since we can equivalently write $\ramp(p) = (|p| + p)/2$, we can rewrite \eqref{eq:sr_deriv1} in the form (for some constants $C_0$ and $C_1$)

\begin{equation}\label{eq:sr_deriv2}
\sr(p) = C_0 + C_1p + \frac{1}{2}\sum_{i=1}^{k}(\sr'_{-}(v_i) - \sr'_{+}(v_i))\cdot (-|p - v_i|).
\end{equation}

Furthermore, since $\sr(p)$ is concave, it will always be the case that $\sr'_{-}(v) - \sr'_{+}(v)$ is non-negative. Let $\lambda_i = \frac{1}{2}(\sr'_{-}(v_i) - \sr'_{+}(v_i))$. Note that $\sum_i \lambda_i = \frac{1}{2}(\sr'_{-}(v_k) - \sr'_{+}(v_1)) = \frac{1}{2}(\sr'(0) - \sr'(1))$. Since $|\sr'(p)| \leq 2$ for any bounded scoring rule $\sr$, it follows that $\sum_{i} \lambda_i \leq 2$. Since $\sr_{v_i}(p) = -|p - v_i|$, equation \eqref{eq:sr_decomp} then immediately follows from \eqref{eq:sr_deriv2}. 

Now, as a consequence of our two earlier observations about $\Reg_{\sr}(\bp, \bx)$, we have that

\begin{equation}
\Reg_{\sr}(\bp, \bx) = \sum_{i=1}^{k}\lambda_i\Reg_{\sr_{v_i}}(\bp, \bx).
\end{equation}

\noindent
It follows that 

\begin{equation}\label{eq:reg_reduc}
\Reg_{\sr}(\bp, \bx) \leq \left(\sum_{i}\lambda_i\right)\cdot\sup_{v \in [0, 1]} \Reg_{\sr_{v}}(\bp, \bx) \leq 2\cdot\VCal(\bp, \bx).
\end{equation}

\noindent
Since $\MaxAgentReg(\bp, \bx) = \Reg_{\sr}(\bp, \bx)$ for some $\sr \in \srset$, we have proved the original inequality.
\end{proof}

\begin{remark}
The proof of Theorem \ref{thm:vcal_approx} extends to any choice of our set $\srset$ of bounded scoring rules (possibly with different constants), as long as i. (some constant multiple of) each of the V-shaped scoring rules belongs to $\srset$ and ii. the derivatives $\sr'(p)$ for any scoring rule $\sr \in \srset$ are absolutely bounded. In fact, if we define $\srset$ to be the collection of scoring rules with the property that $|\sr'(p)|\leq 1$ for all $p \in [0,1]$, then the above proof actually gives an equality between $\MaxAgentReg$ and $\VCal$.
\end{remark}

\begin{remark}
It is instructive to compare Theorem \ref{thm:vcal_approx} above to the results of \citep{hartline2020optimization}. \cite{hartline2020optimization} study (among other things) the problem of finding a (bounded) proper scoring rule for mean estimation that maximizes a specific linear functional (e.g. $\int_{0}^{1}f(p)\sr(p)dp$ for some non-negative valued function $f$). They similarly show for their problem that the optimal scoring rule will always be V-shaped (for a slightly more general definition of V-shaped, where the two sides of the V can have different slopes). Our problem does not fall directly into their framework -- optimizing $\Reg_{\sr}(\bp, \bx)$ requires working with (not necessarily non-negative) linear functionals in both $\sr(p)$ and $\sr'(p)$ instead of just $\sr(p)$ -- but the two settings are similar, and it is possible to reproduce their result by following similar logic as in the above proof. 
\end{remark}

We next examine definition \eqref{eq:vcal} of $\VCal(\bp, \bx)$ in more detail, with the goal of writing it explicitly in terms of $\bp$ and $\bx$. More specifically, for any $v \in [0, 1]$ define $\VReg_{v}(\bp, \bx)$ to be shorthand for $\Reg_{\sr_{v}}(\bp, \bx)$. We have the following explicit formula for $\VReg_{v}(\bp, \bx)$.

\begin{theorem}\label{thm:vreg_char}
Fix $\bp$ and $\bx$. Let $\cP_0$ be the empirical distribution of the $p_{t}$ over rounds $t$ where $x_{t} = 0$; likewise, let $\cP_1$ be the empirical distribution of the $p_{t}$ where $x_{t} = 1$. Then, if $v \leq \beta$, we have that

\begin{equation}\label{eq:vreg_case2}
\VReg_{v}(\bp, \bx) = T \cdot \left (2\beta(1-v) \Pr_{p\sim\cP_{1}}[p < v] - 2(1-\beta)v \Pr_{p\sim\cP_{0}}[p < v] \right) 
\end{equation}

\noindent
and if $v \geq \beta$, we have that

\begin{equation}\label{eq:vreg_case1}
\VReg_{v}(\bp, \bx) = T \cdot \left (2(1-\beta)v \Pr_{p\sim\cP_{0}}[p > v] - 2\beta(1-v) \Pr_{p\sim\cP_{1}}[p > v]\right).
\end{equation}

\end{theorem}
\begin{proof}
To begin, we can expand out the definition of $\VReg_{v}$ to obtain

\begin{equation}\label{eq:vreg1}
\VReg_{v}(\bp, \bx) = \sum_{t=1}^{T} (\sr_{v}(p_t, x_t) - \sr_{v}(\beta, x_t)).
\end{equation}

\noindent
 We can then rewrite \eqref{eq:vreg1} as

\begin{equation}\label{eq:vreg2}
\VReg_{v}(\bp, \bx) = T\left((1-\beta) \E_{p \sim \cP_{0}}\left[\sr_{v}(p, 0)\right] + \beta \E_{p \sim \cP_{1}}\left[\sr_{v}(p, 1)\right] - \sr_{v}(\beta)\right).
\end{equation}

Now, note that the bivariate form of $\sr_{v}$ is given by $\sr_{v}(p, 0) = v \cdot \sgn(p - v)$ and $\sr_{v}(p, 1) = (1-v) \cdot \sgn(v - p)$ (where we define $\sgn(x) = 1$ for $x > 0$, $-1$ for $x < 0$, and $0$ at $x = 0$). Substituting these into \eqref{eq:vreg2} (and applying the identity $\E[\sgn(X)] = 2\Pr[X > 0] - 1$), we arrive at \eqref{eq:vreg_case1} and \eqref{eq:vreg_case2}. 
\end{proof}

\begin{remark}
There is one technical subtlety in the above theorem, which is that the values of the bivariate form of the V-shaped scoring rule $\sr_v(p, x)$ are not uniquely defined when $p = v$ -- above we set $\sr_{v}(v, 0) = \sr_{v}(v, 1) = 0$, but another valid choice is $\sr_{v}(v, 0) = \lambda v$ and $\sr_{v}(v, 1) = \lambda(1-v)$ for any $\lambda \in [-1, 1]$. This choice \textit{does} affect the value of $\VReg_{v}(\bp, \bx)$ when $v$ is equal to some $p_t$. 

However, one consequence of Theorem \ref{thm:vreg_char} is that $\VReg_{v}(\bp, \bx)$ is a piece-wise linear function of $v$, with breakpoints at values taken by $p_t$. Because there are only finitely many such values, to compute the supremum of $\VReg_{v}$ over the interval $[0, 1]$, it suffices to evaluate $\VReg_{v}$ only at non-breakpoints (where the above formulae are valid independent of our choice of $\sr_{v}(v, x)$), so the value of $\VCal(\bp, \bx)$ is independent of these details. 
\end{remark}

\subsection{Some examples of V-calibration}
\label{sec:v-calibration-examples}

To gain an intuition for V-calibration and V-regret, it is useful to consider some examples. Below we work through three examples: the first a forecast with high ($\Omega(T)$) V-calibration error, the second a perfectly calibrated forecast (in the sense of regular calibration), and finally, an example of a forecast with high calibration error and low ($o(T)$) V-calibration error. In all three of these examples, the underlying sequence of binary events will be the same; we will have $x_t = 1$ for $t \in [1, T/2]$, and $x_t = 0$ for $t \in [T/2+1, T]$ (so the base rate $\beta = 1/2$). Only the forecasts $p_t$ will change.

\paragraph{Example 1}
We begin with the example from Theorem \ref{thm:low_brier}, where a sequence of predictions with low Brier score nonetheless has agent regret for a specific agent (so we should expect it to have high V-calibration error). For this example, $\beta = 1/2$, $\cP_0$ is a Bernoulli distribution with mean $1/4$, and $\cP_{1}$ is a Bernoulli distribution with mean $3/4$. We can then apply Theorem \ref{thm:vreg_char} to work out that

\begin{equation}\label{eq:vcal_ex1}
\frac{1}{T} \cdot \VReg_{v}(\bp, \bx) = \begin{cases}
\frac{1}{4} - v & \mbox{ if } v \in [0, 1/2] \\
v - \frac{3}{4} & \mbox{ if } v \in [1/2, 1].
\end{cases}
\end{equation}

From \eqref{eq:vcal_ex1}, we can see that there are values of $v$ where $\VReg_{v} = \Omega(T)$. For example, when $v = 0.9$, $\VReg_{v} = 0.15T$ (and indeed, this corresponds to the gap we obtain in Theorem \ref{thm:low_brier}). On the other hand, for $v \in [1/4, 3/4]$, $\VReg_{v} \leq 0$ -- for agents corresponding to these scoring rules, this sequence of predictions does in fact lead to low regret. The maximum value of $\VReg_{v}$ is attained at $0.25T$ (when $v \in \{0, 1\}$), so for this example $\VCal_{v} = 0.25T = \Omega(T)$. 

\paragraph{Example 2} Second, we will consider a perfectly calibrated sequence of predictions, where $p_t = x_t$ for all $t$. For this example, $\beta = 1/2$, $\cP_{0}$ is a singleton distribution supported at $0$, and $\cP_{1}$ is a singleton distribution supported at $1$. By applying Theorem \ref{thm:vreg_char}, we can work out that

\begin{equation}\label{eq:vcal_ex2}
\frac{1}{T} \cdot \VReg_{v}(\bp, \bx) = \begin{cases}
- v & \mbox{ if } v \in [0, 1/2] \\
v - 1 & \mbox{ if } v \in [1/2, 1].
\end{cases}
\end{equation}

Unsurprisingly, $\VReg_{v}(\bp, \bx) \leq 0$ for all $v \in [0, 1]$, and we have that $\VCal(\bp, \bx) = 0$. In fact, we have that $\VReg_{v}(\bp, \bx) = -\Omega(T)$ for all $v$ except $v = 1/2$. One way to view this is as saying that for almost all scoring rules, following this calibrated sequence of predictions will cause you to significantly outperform (by at least $\Omega(T)$) the base rate forecaster. 

\paragraph{Example 3}\label{ex:ftl} Finally, we consider a slightly more involved example. The predictions $p_t$ will be generated by the \textit{empirical average forecaster}, who always predicts the current historical average of $x_{t}$. Specifically, for $t \in [1, T/2]$ we will set $p_t = 1$, and for $t > T/2$ we will set $p_t = (T/2)/t$. 

For this example, we again have $\beta = 1/2$. The distribution $\cP_{1}$ is simply the singleton distribution at $1$. However, the distribution $\cP_{0}$ is slightly more complex; it is the uniform distribution over the set of numbers of the form $(T/2)/t$ for $t \in [T/2 +1, T]$. As $T$ approaches infinity, the CDF of $\cP_{0}$ approaches the function $F_{0}(q) = \max(0, 2 - 1/q)$. To see this, note that in order for $(T/2)/t < q$, we must have $t/(T/2) > (1/q)$; since $t/(T/2)$ is (in the limit) distributed uniformly in $[1, 2]$, this occurs with probability $2 - (1/q)$ (as long as $1/q \leq 2$). 

Applying Theorem \ref{thm:vreg_char}, we can then work out that for $v \leq 1/2$, 

$$\lim_{T \rightarrow \infty} \frac{1}{T}\cdot\VReg_{v}(\bp, \bx) = (1-v)\Pr_{\cP_{1}}[p < v] - v\Pr_{\cP_{0}}[p<v] = 0,$$

\noindent
and similarly, for $v \geq 1/2$ we have that

$$\lim_{T \rightarrow \infty} \frac{1}{T}\cdot\VReg_{v}(\bp, \bx) = v\Pr_{\cP_{0}}[p > v] - (1-v)\Pr_{\cP_{1}}[p>v] = v\left(1 - \left(2 - \frac{1}{v}\right)\right)- (1-v) = 0.$$

That is, (in the limit) $\VReg_{v}(\bp, \bx)$ is identically zero (and thus so is $\VCal(\bp, \bx)$). One consequence of this (by Theorem \ref{thm:vcal_approx}) is that this sequence of forecasts performs \textit{exactly} as well as the base rate forecaster when measured with respect to \textit{any} scoring rule. We will see some explanation for this in Section \ref{sec:algs}, when we discuss algorithms for V-calibration.

Note that, despite having zero V-calibration, this example is \textit{not} calibrated in the standard sense. In particular, every prediction made in the latter half of the time horizon appears uniquely and has an error of at least $1/2$, so $\Cal(\bp, \bx) \geq 0.5T$. In fact, as we show in Appendix \ref{sec:other-calibration}, not only is this sequence of predictions not calibrated according to our definition of calibration error, it is also far from being calibrated for several existing notions of smooth / approximate calibration.

\subsection{An algorithm for online V-calibration}
\label{sec:algs}

We now switch our attention to online procedures for producing V-calibrated forecasts. Since any regularly calibrated forecaster is also V-calibrated (Theorem \ref{thm:cal_suffices}), we can apply any calibrated forecasting procedure to obtain a V-calibrated procedure. Furthermore, by Theorem \ref{thm:cal_suffices}, if such a procedure guarantees calibration error $R(T)$, it also guarantees V-calibration error $O(R(T))$. 

However, although there exist procedures for producing forecasts with $o(T)$ calibration error \citep{foster1998asymptotic,blum2007external}, the best-known procedures incur $O(T^{2/3})$ calibration error and are somewhat non-intuitive (in general, they involve repeatedly solving some sort of fixed-point problem). Here we give a simple, efficient procedure for V-calibration that guarantees $O(\sqrt{T})$ V-calibration error, which is asymptotically tight.

We begin by describing our algorithm, which we call $\ForecastHedge$:

\begin{algorithm}[H]
\caption{$\ForecastHedge$:}
\label{alg:forecast_hedge}
\begin{algorithmic} 
\State Let $S(x) = e^{x}/(e^{x} + e^{-x})$ and $\eta = 1/\sqrt{T}$.
\State Predict $p_1 = 1/2$ and observe $x_1$. Set $\hat{x}_{1} = x_1$.
\State For $t=2$ to $T$:
\State \quad Sample prediction $p_t \in [0,1]$ such that $\Pr[p_t \leq v] = S\left(\eta(t-1)(v - \hat{x}_{t-1})\right), \forall v \in [0,1)$
\State \quad Observe $x_t$ and update $\hat{x}_{t} = \frac{1}{(t-1)}\sum_{s=1}^{t-1}x_{s}$
\end{algorithmic}
\end{algorithm}

Note that since $S(x)$ is an increasing function in $x$ bounded between $0$ and $1$, this describes a valid probability distribution. 

\begin{theorem}\label{thm:vcal_alg}
For any sequence $\bx$ of events, $\ForecastHedge$ produces a sequence of predictions $\bp$ which has an expected V-calibration error of at most $O(\sqrt{T})$, i.e., $\E[\VCal(\bp, \bx)] = O(\sqrt{T})$. 
\end{theorem}

Before we proceed to the proof of Theorem \ref{thm:vcal_alg}, we present some high-level intuition for why $\ForecastHedge$ works. We begin by considering the simpler problem of how to design a learning algorithm that minimizes the agent regret for a specific agent. For this, we can simply employ the classic Hedge algorithm (see \cite{arora2012multiplicative}).

\begin{lemma}[Hedge algorithm]\label{lem:hedge}
\sloppy{Fix a utility function $u: \cA \times \{0, 1\} \rightarrow [-1, 1]$ and set $\eta = \sqrt{(\log |\cA|) / T}$. Consider the agent that, at round $t$, plays the action $a \in \cA$ with probability proportional to $\exp(\eta\sum_{s=1}^{t-1} u(a, x_s))$. Then, in expectation over the randomness of the algorithm, this agent has at most $O(\sqrt{T\log|\cA|})$ regret:}

$$\E\left[\sum_{t=1}^{T}u(a_{\beta}, x_t) - \sum_{t=1}^{T}u(a_t, x_t)\right] = O(\sqrt{T\log |\cA|}).$$
\end{lemma}

Note that instead of specifying a forecast $p_t$ for $x_t$ (which the Agent then best responds to), the Hedge algorithm directly specifies the distribution over actions that the Agent should play at time $t$. At a high level, in $\ForecastHedge$ we sample $p_t$ in such a way to incentivize exactly the same distribution over actions as the Agent would play if they were running Hedge. More accurately, this is not true for every possible agent, but \textit{it is true for the family of Agents that correspond to V-shaped scoring rules} (which is sufficient to minimize V-calibration error). The statement of Theorem \ref{thm:vcal_alg} follows from this fact, modulo some technical complexity due to the fact that we want to bound the expectation of a maximum over infinitely many random variables (for which we apply a variant of the DKW inequality we develop in Appendix \ref{sec:dkw-noniid}).\\

\begin{proof}[Proof of Theorem \ref{thm:vcal_alg}]
For a $v \in [0, 1]$, consider the utility function $u_v: \{0, 1\} \times \{0, 1\}$ defined via $u_v(a, x) = (-1)^{a}(v-x)$. Note that for this utility function, we have that $\AgentReg_{u_v} = \Reg_{\sr_v}$. 

An Agent with utility function $u_{v}$ plays action $a_t = 0$ when $p_t \leq v$ and action $a_t = 1$ when $p_t \geq v$. If they follow the sequence of predictions generated by $\ForecastHedge$, they play action $a_t = 0$ with probability

$$\Pr[a_t = 0] = \Pr[p_t \leq v] = S\left(\eta(t-1)(v - \hat{x}_{t-1})\right).$$

On the other hand, if this Agent instead followed the Hedge algorithm of Lemma \ref{lem:hedge}, they would play action $a_t = 0$ with probability proportional to $\exp(\eta\sum_{s=1}^{t-1} u_v(0, x_s)) = \exp(\eta(t-1)(v - \hat{x}_{t-1}))$, and action $a_t = 1$ with probability proportional to $\exp(\eta\sum_{s=1}^{t-1} u_v(1, x_s)) = \exp(-\eta(t-1)(v - \hat{x}_{t-1}))$. It follows that the Agent following Hedge also plays action $a_t = 0$ with probability

$$\Pr[a_t = 0] = \frac{\exp(\eta(t-1)(v - \hat{x}_{t-1}))}{\exp(\eta(t-1)(v - \hat{x}_{t-1})) + \exp(-\eta(t-1)(v - \hat{x}_{t-1}))} = S\left(\eta(t-1)(v - \hat{x}_{t-1})\right).$$

These two agents have exactly the same behavior in response to any sequence of events $\bx$. By the guarantee of Lemma \ref{lem:hedge}, it follows that $\E[\VReg_v(\bp, \bx)] = \E[\AgentReg_{u_v}(\bp, \bx)] = O(\sqrt{T})$.\\

As a final step we want to go from a bound on $\sup_v \E[\VReg_v(\bp, \bx)]$ to a bound  on $\E[\VCal(\bp, \bx)] = \E[\sup_v \VReg_v(\bp, \bx)]$. For that we will use the uniform convergence bound for monotone functions established in Theorem \ref{thm:dkw-noniid} in Appendix \ref{sec:dkw-noniid}. To apply this theorem, fix a sequence $\bx$ and define $T_0 = \{t; x_t=0\}$ and  $T_1 = \{t; x_t=1\}$. Now, for the base rate $\beta$ we can rewrite:

\begin{eqnarray}
\VReg_v(\bp, \bx) &=& \sum_{t \in T_0} \ell_v(p_t,0) + \sum_{t \in T_1} \ell_v(p_t,1) - \sum_{t \in T} \ell_v(\beta,x_t) \nonumber\\
&=& \underbrace{\sum_{t \in T_0} [v+ \ell_v(p_t,0)]}_{A_v} + \underbrace{\sum_{t \in T_1} [1-v+ \ell_v(p_t,1)]}_{B_v} \underbrace{- \sum_{t \in T} \ell_v(\beta,x_t) - v \abs{T_0} - (1-v) \abs{T_1}}_{C_v} \label{eq:alg_vreg_decomp}
\end{eqnarray}

The function $\ell_v(p,0) = -v$ for $v \leq p$ and $\ell_v(p,0) = v$ otherwise. Hence $v+ \ell_v(p_t,0)$ is monotone non-decreasing in $v$ and has range $[0,2]$. Similarly  $1-v+ \ell_v(p_t,1)$  is monotone non-increasing in $v$ and has range $[0,2]$. Finally note that the random variables $p_t$ are independent since we choose the distribution of the prediction only based on the historical average $\hat{x}_{t-1}$ and not on the previous predictions. Hence $v+ \ell_v(p_t,0)$ and $1-v+ \ell_v(p_t,1)$ are independent random monotone functions of bounded range, satisfying the conditions of Theorem \ref{thm:dkw-noniid}. Using the shorthand $A_v$, $B_v$ and $C_v$ defined in \eqref{eq:alg_vreg_decomp}, we have:
$$ \E \sup_v | A_v - \E[A_v]| \leq O(\sqrt{T_0}) \qquad  \E \sup_v | B_v - \E[B_v]| \leq O(\sqrt{T_1}) $$
Observe also that $C_v$ is not random. Now we can bound $\VCal$ as follows:
$$\begin{aligned}
\E[\VCal(\bp, \bx)] &= \E[\sup_v (A_v + B_v + C_v)] \leq \E[\sup_v ( C_v + \E A_v + \E B_v  +\abs{A_v + B_v - \E[A_v + B_v ] })] \\ & \leq \sup_v ( C_v + \E A_v + \E B_v ) + \E \sup_v (\abs{A_v  - \E A_v}) +\E \sup_v (\abs{B_v  - \E B_v}) \\ 
& = \sup_v \E[\VReg_v(\bp, \bx)] + \E \sup_v (\abs{A_v  - \E A_v}) +\E \sup_v (\abs{B_v  - \E B_v}) \\
& \leq O(\sqrt{T}) + O(\sqrt{T_0}) + O(\sqrt{T_1}) = O(\sqrt{T})
\end{aligned}$$
\end{proof}

\begin{remark}
It is interesting to reexamine Example 3 in light of the guarantees provided by Theorem \ref{thm:vcal_alg}. In the limit as $T \rightarrow \infty$, the predictions made by $\ForecastHedge$ converge to the predictions made by the empirical average forecaster (for the specific sequence of events $x_t$ in the example), and therefore the empirical average forecaster should have low V-calibration error in the limit. This also helps explain why the V-regret of the empirical average forecaster is asymptotically 
uniformly zero in this example: it can be shown that an agent running Hedge against a ``stable'' loss sequence (one where the best arm in hindsight does not change too often) will end up with utility close to that of the optimal arm (and hence near-zero regret).

That said, the empirical average forecaster does not, in general, result in low V-calibration error. In fact, no deterministic forecasting procedure can result in low V-calibration error (if the Forecaster is running a deterministic algorithm, the Adversary can always select $x_t = 0$ when $p_t \geq 0.5$ and $x_t = 1$ otherwise). 
\end{remark}

\subsection{Calibration and swap regret}

We conclude our discussion of the binary forecasting problem with a discussion of how U-calibration relates to swap regret. A U-calibrated sequence of forecasts ensures that an agent consuming these forecasts has low \textit{external regret} -- the gap between their cumulative utility and the cumulative utility of their best-in-hindsight action -- regardless of what their utility function is. One might also wish to limit an agent's \textit{swap regret} -- the gap between their cumulative utility and their counter-factual utility if they had applied a fixed swap function $\pi:\cA \rightarrow \cA$ to their sequence of actions. Formally, we can define agent swap regret analogously to the external regret of an agent as follows:

\begin{equation}\label{eq:agent_swap_reg}
\AgentSwapReg_u(\bp, \bx) = \max_{\pi:\cA\rightarrow \cA} \sum_{t=1}^{T} u(\pi(a(p_t)), x_t) - \sum_{t=1}^{T} u(a(p_t), x_t).
\end{equation}

One motivation for studying swap regret is that calibrated forecasts imply sublinear swap regret for any agent. In fact, one of the first applications of online calibrated forecasting was to design low swap-regret algorithms for agents in games and hence game dynamics that converge to a correlated equilibrium \citep{foster1997calibrated}. This is captured in the following analogue of Theorem \ref{thm:cal_suffices} (with a very similar proof, included in Appendix \ref{app:omitted}).

\begin{theorem}\label{thm:cal_suffices_for_swap}
For any sequence of $T$ binary events $\bx$, predictions $\bp$, and bounded agent (with utility $u$), we have that $\AgentSwapReg_u(\bp, \bx) \leq 4\Cal(\bp, \bx)$. In particular, if $\bp$ and $\bx$ satisfy $\Cal(\bp, \bx) = o(T)$, then for any bounded agent, $\AgentSwapReg_u(\bp, \bx) = o(T)$. 
\end{theorem}

As with $\AgentReg$, we can ask if calibration is truly necessary here, or if there is a weaker analogue of calibration (\`a la U-calibration) which would suffice to minimize agent swap regret. Interestingly, we show the answer is essentially no -- for any miscalibrated sequence of forecasts $\bp$ for $\bx$, there is an agent which incurs high swap regret if they follow these forecasts. That is, (ordinary) calibration has the same relation to agent swap regret that U-calibration does to agent external regret.

To prove this, we will find it easier to work with the following $L_{2}$-variant of calibration error.

\begin{equation}\label{eq:l2-calibration}
\Cal_{2}(\bp, \bx) = \sum_{p \in [0, 1]} n_p\left(p - \frac{m_p}{n_p}\right)^2.
\end{equation}

Regular calibration error and $L_2$-calibration error are related by the following inequality.

\begin{lemma}\label{lem:l2-vs-l1-calibration}
For any $\bp$ and $\bx$,

$$\left(\frac{\Cal(\bp, \bx)}{T}\right)^2 \leq \frac{\Cal_{2}(\bp, \bx)}{T} \leq \frac{\Cal(\bp, \bx)}{T}.$$
\end{lemma}
\begin{proof}
The right inequality follows since $\left(p - \frac{m_p}{n_p}\right)^2 \leq |p - \frac{m_p}{n_p}|$. The left inequality follows from the following application of Cauchy-Schwartz:

$$\left(\sum_{p \in [0, 1]}n_p\left(p - \frac{m_p}{n_p}\right)^2\right)\left(\sum_{p \in [0,1]}n_p\right) \geq \left(\sum_{p\in[0,1]} n_p\left|p - \frac{m_p}{n_p}\right|\right)^2.$$
\end{proof}

In particular, by Lemma \ref{lem:l2-vs-l1-calibration}, a forecaster has $\Omega(T)$ $L_2$-calibration error iff they have $\Omega(T)$ regular calibration error. We can now prove the following theorem, which shows that we can always construct an agent where $\AgentSwapReg_{u}(\bp, \bx) \geq \Cal_{2}(\bp, \bx)$. At a high level, we will show that if we take our agent to themselves be a forecaster rewarded according to the Brier score, their swap regret is equal to $L_{2}$-calibration error.

\begin{theorem}\label{thm:cal_is_necessary}
For any sequence of $T$ predictions $\bp$ and binary events $\bx$, there exists a bounded utility function $u$ such that $\AgentSwapReg_{u}(\bp, \bx) \geq \Cal_{2}(\bp, \bx)$.
\end{theorem}
\begin{proof}
Consider the agent\footnote{We define this agent as having infinitely many actions (one for each element in $[0, 1]$). However, note that we can reduce this to a finite number of actions by restricting $\cA$ to values that appear as either $p_t$ or $\pi(p_t)$.} with $\cA = [0, 1]$ and $u(a, x) = -(a-x)^2$. Note that for this agent, $a(p) = p$ and $\wtu(p, x) = -(p-x)^2$. Furthermore, the best swap function $\pi:\cA\rightarrow\cA$ is the one which sends each $p$ in the support of $\bp$ to $\pi(p) = \frac{m_p}{n_p}$. Now, we have that

\begin{eqnarray*}
\AgentSwapReg_{u}(\bp, \bx) &=& \sum_{t=1}^{T} u(\pi(a(p_t)), x_t) - \sum_{t=1}^{T} u(a(p_t), x_t) \\
    &=& \sum_{t=1}^{T} (p_t - x_t)^2 - (\pi(p_t) - x_t)^2 \\
    &=& \sum_{p \in [0, 1]} \sum_{t; p_t = p}\left((p - x_t)^2 - \left(\frac{m_p}{n_p} - x_t\right)^2\right) \\
    &=& \sum_{p \in [0, 1]} \sum_{t; p_t = p}\left(p - \frac{m_p}{n_p}\right)\left(p + \frac{m_p}{n_p} - 2x_t\right) \\
    &=& \sum_{p \in [0, 1]} \sum_{t; p_t = p}n_p\left(p - \frac{m_p}{n_p}\right)\left(p + \frac{m_p}{n_p} - 2\frac{m_p}{n_p}\right) \\
    &=& \sum_{p \in [0, 1]} \sum_{t; p_t = p}n_p\left(p - \frac{m_p}{n_p}\right)^2 = \Cal_{2}(\bp, \bx).
\end{eqnarray*}
\end{proof}

What does this imply for the use of U-calibration as a forecasting metric? In particular, Theorem \ref{thm:cal_is_necessary} implies that forecasts which are U-calibrated but not calibrated (e.g. the third example of Section \ref{sec:v-calibration-examples}) will have high agent swap regret for some agent. Does this mean we should insist that all forecasts are not merely U-calibrated but also calibrated?

Ultimately, it seems like the answer to this question should depend on what, specifically, is being forecasted. Swap regret tends to be a useful quantity for agents to minimize in \textit{strategic} settings -- for example, it leads to convergence to correlated equilibria \citep{foster1997calibrated} and low swap regret algorithms cannot be dynamically manipulated by other players the way low external regret algorithms can \citep{deng2019strategizing, mansour2022strategizing}. So, in settings where the event being forecasted is controlled by a strategic agent who cares about the action the agent takes (e.g., forecasting the strategies of other players in a game, as in \cite{foster1997calibrated}), it may make sense to insist on calibrated forecasts. But in other settings where the outcome-generating procedure is non-strategic (e.g., the weather) it is not obvious what benefits calibrated forecasts provide over U-calibrated forecasts.

\section{Multiclass U-calibration}
\label{sec:multiclass_u_cal}

\subsection{Multiclass forecasting}\label{sec:multiclass_definitions}

In this section, we consider extensions to the setting of \textit{multiclass forecasting}, where each event has one of $K$ possible outcomes and the forecaster's predictions take the form of distributions over these $K$ outcomes. 

We begin by reviewing how the definitions in the binary case generalize to the multiclass setting. Our scoring rules now take the form $\sr: \Delta_{K} \times [K] \rightarrow \Rset$, where each event $x$ lies in $[K]$ and each prediction $p$ belongs to the $K$-simplex $\Delta_K$. As before, a scoring rule is proper if $\E_{x \sim p}[\sr(p, x)] \leq \E_{x \sim p}[\sr(p', x)]$ for any $p' \neq p$, and is strictly proper if this inequality is strict. Similarly, for $p, q \in \Delta_K$, we write $\sr(p; q) = \E_{x \sim q}[\sr(p, x)]$, and define the univariate form $\sr(p) = \sr(p; p)$. As in the binary case, the univariate forms of scoring rules still correspond to concave functions (now over $\Delta_K$).

\begin{lemma}\label{lem:uni_to_bi}
For any scoring rule $\sr$, the univariate form $\sr(p)$ is a concave function over $\Delta_K$. Moreover, given any concave function $f: \Delta_{K} \rightarrow \Rset$, there exists a scoring rule $\sr$ such that $\sr(p) = f(p)$. Finally, if $\sr(p)$ is differentiable, it is possible to recover the bivariate form of $\sr$ from the univariate form via the equality\footnote{For convenience, we will abuse notation by identifying the set $[K]$ of outcomes with the set $\{e_1, e_2, \dots, e_K\}$ of unit vectors in $\Rset^{K}$. This allows us to write expressions like $(x-p)$ in place of $(e_x - p)$, and more closely matches the notation of the binary outcome case.}

\begin{equation}\label{eq:uni_to_bi}
\sr(p, x) = \sr(p) + \langle x - p, \nabla\sr(p) \rangle.
\end{equation}
\end{lemma}

As in the binary case, we will restrict our attention to the set of bounded scoring rules $\cL$ containing all scoring rules taking on values bounded within the interval $[-1, 1]$. Again, this implies bounds on the gradient $\nabla \sr(p)$: in particular, it is the case that for any $p, q, q' \in \Delta_K$, $\langle q-q', \nabla\sr(p) \rangle \leq ||q - q'||_1$ (see Corollary \ref{cor:multiclass_derivative_bound}). 

The forecasting game (involving an Adversary, a Forecaster, and an Agent) remains essentially the same in the multiclass setting. Again, the Adversary selects a sequence of outcomes $x_{t} \in [K]$, the Forecaster produces a prediction $p_{t} \in \Delta_{K}$ for $x_t$ based on the previous predictions and outcomes, and the Agent (equipped with a utility function $u: \cA \times [K] \rightarrow [-1, 1]$) observes the prediction $p_{t}$ and plays the action $a_t$ that maximizes their expected utility $\E_{x \sim p_t}[u(a_t, x)]$. We again employ as our baseline the base rate forecast $\beta = \frac{1}{T}\sum_{t} x_t$ (which is now an element of $\Delta_K$). 

It is fairly straightforward to generalize scoring rule regret $\Reg_{\ell}$ and agent regret $\AgentReg_{u}$ to the multiclass setting (in particular, definitions \eqref{eq:sr_reg} and \eqref{eq:agent_reg} generalize as written). It is less clear what the definition of multiclass calibration should be. Here we define it as follows, where we look at the average $\ell_1$ prediction error over all rounds where the Forecaster makes exactly the same prediction $p_t$ (which matches the definition in \cite{foster1997calibrated}):

\begin{equation}\label{eq:multiclass_calibration}
\Cal(\bp, \bx) = \sum_{p \in \Delta_K} \left\Vert \sum_{t\,\mid\,p_{t} = p} (p-x_t) \right\Vert_1,
\end{equation}

Under this definition, we have the following analogue of Theorem \ref{thm:cal_suffices}:

\begin{theorem}\label{thm:multiclass_cal_suffices}
For any sequence of $T$ multiclass events $\bx$, $T$ multiclass predictions $\bp$, and bounded agent (with utility $u$), we have that $\AgentReg_u(\bp, \bx) \leq 2\Cal(\bp, \bx)$. In particular, if $\bp$ and $\bx$ satisfy $\Cal(\bp, \bx) = o(T)$, then for any bounded agent, $\AgentReg_u(\bp, \bx) = o(T)$. 
\end{theorem}

The choice of $\ell_1$ norm in the definition \eqref{eq:multiclass_calibration} of multiclass calibration is fairly arbitrary -- replacing it with a different $\ell_p$ norm simply decreases the calibration error by at most a factor of $K$ (so in particular, it is still true that sublinear calibration error implies sublinear agent regret under different $\ell_p$ norms). We briefly note that other weaker notions of multiclass calibration are often used in practice, e.g. ``one-vs-all'' notions which look at the maximum calibration error in each dimension individually \citep{johansson2021calibrating}. It turns out that these forms of calibration do not have the property of guaranteeing sublinear agent regret (we will see a counterexample in Section \ref{sec:one_vs_all}). 

\subsection{Computing multiclass U-calibration error}
\label{sec:computing_agent_cal}

Although (multiclass) calibration guarantees sublinear agent regret, it is not \emph{necessary} to guarantee sublinear agent regret. For this, we would like to minimize the U-calibration error, which (just as in the binary case) is defined to equal $\MaxAgentReg(\bp, \bx) = \sup_{\ell \in \cL} \Reg_{\ell}(\bp, \bx)$. 

In the binary case, we demonstrated that instead of minimizing regret for all bounded scoring rules in $\cL$, it suffices to minimize regret for the specific class of V-shaped scoring rules. In the multiclass setting, it is not clear what the correct analogue of V-shaped scoring rules should be (we explore this question in Section \ref{sec:multiclass_v_cal}). Instead, in this section we will describe how to directly (and efficiently) evaluate $\MaxAgentReg(\bp, \bx)$ for a given sequence of predictions and outcomes by solving a specific convex program.

\begin{theorem}\label{thm:agent_cal_alg}
Given a sequence of $T$ multiclass (taking on $K$ values) outcomes $\bx$, $T$ multiclass predictions $\bp$, and an $\varepsilon > 0$, there is an algorithm that computes a bounded scoring rule $\ell \in \cL$ with the property that $\Reg_{\ell}(\bp, \bx) \geq \MaxAgentReg(\bp, \bx) - \varepsilon$ in time $\poly(T, K, \log \frac{1}{\varepsilon})$. 
\end{theorem}
\begin{proof}
Note that for a fixed sequence of outcomes $\bx$ and predictions $\bp$, the value of $\Reg_{\sr}(\bp, \bx)$ is determined by the values of $\sr(p_t, x)$ for each $x \in [K]$ and $t \in [T]$, and also the values of $\sr(\beta, x)$ at the base rate prediction for each $x \in [K]$. We can therefore consider this task an optimization problem over the set $\cY \subseteq [-1, 1]^{K(T+1)}$ containing the $K(T+1)$ tuples of values consistent with an actual scoring rule $\sr \in \cL$ (i.e., if $y \in Y$, then $y_{t, x} = \sr(p_t, x)$ for some $t \in [T+1]$ and $x \in [K]$, taking $p_{t+1} = \beta$ for convenience).

We now make the following two claims, from which it follows that there is an efficient optimization algorithm for this problem with the guarantees of the theorem statement. 

\begin{enumerate}
    \item The set $\cY$ is convex.
    \item There is an efficient membership oracle for $\cY$. Moreover, if $y \in \cY$, this oracle can efficiently construct a scoring rule $\sr$ compatible with $y$.
\end{enumerate}

The first fact above follows from the fact that the set of bounded scoring rules is convex; for any $\ell, \ell' \in \cL$, $\lambda \ell + (1-\lambda) \ell' \in \cL$ for any $\lambda \in [0, 1]$. To prove the second fact, we claim that given a $y \in \cY$ it suffices to check if any of the $(T+1)^2$ linear inequalities $\langle y_{t}, p_t \rangle \leq \langle y_{t}, p_{t'}\rangle$ are violated for $t, t' \in [T+1]$. Note that if $y$ is consistent with a scoring rule $\sr$, this inequality is equivalent to the statement that $\sr(p_t) \leq \sr(p_t;p_{t'})$, which is a requirement for $\sr$ to be a proper scoring rule.

On the other hand, if all these inequalities hold, construct the candidate scoring rule $\sr_y$ with multivariate form $\sr_{y}(p, x) = y_{\tau(p), x}$, where $\tau(p) = \arg\min_{t \in [T+1]} \langle y_{t}, p \rangle$; since the previous inequalities hold, it is true that $\tau(p_t) = t$ and thus $\sr_{y}(p_t, x) = y_{t, x}$. The univariate form of this scoring rule is then given by $\sr_y(p) = \min_{t \in [T+1]} \langle y_{t}, p \rangle$. This is a concave function bounded in $[-1, 1]$, so it generates a valid bounded multiclass scoring rule per Lemma \ref{lem:uni_to_bi}. 
\end{proof}

\subsection{Barriers to multiclass V-calibration}
\label{sec:multiclass_v_cal}

Inspired by the reduction in the binary setting from $\MaxAgentReg$ to $\VCal$, we might ask if there is a representative family of multiclass scoring rules (similar to V-shaped scoring rules) such that it suffices to ensure that our forecasts have low regret with respect to the scoring rules in this family. In particular, we propose the following (somewhat loosely defined) open question.

\begin{question}\label{question:multiclass_v_cal}
Is there a ``nice'' (e.g., low-parameter) family of bounded multiclass (over $K$ outcomes) scoring rules $\cL' \subseteq \cL$ and a constant $C_K > 0$ such that for any sequence of $T$ outcomes $\bx$ and predictions $\bp$,

$$\sup_{\ell \in \cL'} \Reg_{\ell}(\bp, \bx) \geq C_K \cdot \sup_{\ell \in \cL} \Reg_{\ell}(\bp, \bx).$$
\end{question}

In this section we present two barriers to resolving this question. We first show (in Section \ref{sec:one_vs_all}) that any such family cannot treat different dimensions completely independently -- in other words, it is not enough to simply be calibrated with respect to each individual outcome (in a binary sense) individually. We then argue (in Section \ref{sec:generating_scoring_rules}) that such a family of scoring rules cannot form a positive linear basis for the full set of bounded scoring rules (a property that V-shaped scoring rules possesses and that we take advantage of in the proof of Theorem \ref{thm:vcal_approx}). 

\subsubsection{Treating outcomes independently}
\label{sec:one_vs_all}

It is tempting to try to reduce the problem of multiclass forecasting to binary forecasting. In particular, one natural hypothesis is that for a sequence of multiclass predictions to be U-calibrated, it is enough for this sequence of predictions to be U-calibrated for each individual outcome (deriving from this sequence of multiclass predictions a sequence of binary predictions of the form ``will outcome $i$ happen or not?''). Indeed, this is the basis of ``one-vs-all'' methods for standard multiclass calibration \citep{johansson2021calibrating}.

This hypothesis turns out to be false. In particular, there exist sequences of multiclass predictions (even in the case of $K=3$ classes) that are \emph{perfectly calibrated} (and hence perfectly U-calibrated) with respect to each outcome, but that have $\Omega(T)$ U-calibration error as multiclass forecasts. 

Given a sequence of multiclass predictions $\bp = (p_1, \dots, p_T) \in \Delta_K^{T}$, for each outcome $i \in [K]$ let $\bp^{(i)} = (p^{(i)}_1, \dots p^{(i)}_T) \in [0,1]^T$ be the sequence of binary predictions formed via $p^{(i)}_t = (p_t)_i$. Similarly, given a sequence of a multiclass outcomes $\bx = (x_1, x_2, \dots, x_T) \in [K]^T$, for each outcome $i \in [K]$ let $\bx^{(i)} = (x^{(i)}_1, \dots, x^{(i)}_T)$ be the sequence of binary events formed via $x^{(i)}_t = \bm{1}(x_{t} = i)$.

\begin{theorem}\label{thm:binary_no_multiclass}
There exists a sequence of $T$ multiclass (for $K=3$) events $\bx$ and predictions $\bp$ such that $\Cal(\bp^{(i)}, \bx^{(i)}) = 0$ for each $i \in [K]$, but $\MaxAgentReg(\bp, \bx) = \Omega(T)$. 
\end{theorem}
\begin{proof}
We begin by specifying the sequences $\bx$ and $\bp$. We will divide time into 9 ``epochs'' of equal size $T/9$ rounds each. Within each epoch, $p_t$ and $x_t$ are constant, and we write down the specific schedule of these variables in Table \ref{table:binary_no_multiclass}. 

\begin{table}\label{table:binary_no_multiclass}
\begin{center}
\def\arraystretch{1.5}
\begin{tabular}{ c || c | c | c | c | c | c | c | c | c}
 $x_t$ & 1 & 1 & 1 & 2 & 2 & 2 & 3 & 3 & 3 \\ 
 \hline
 $p_t$ & $(\frac{2}{3}, 0, \frac{1}{3})$ & $(\frac{2}{3}, 0, \frac{1}{3})$ & $(\frac{1}{3}, 0, \frac{2}{3})$ & $(\frac{1}{3},\frac{2}{3}, 0)$ & $(\frac{1}{3},\frac{2}{3}, 0)$ & $(\frac{2}{3},\frac{1}{3}, 0)$ & $(0, \frac{1}{3},\frac{2}{3})$ & $(0, \frac{1}{3},\frac{2}{3})$ & $(0, \frac{2}{3},\frac{1}{3})$
\end{tabular}
\caption{Sequence of multiclass predictions and events for Theorem ~\ref{thm:binary_no_multiclass}}
\end{center}
\end{table}

It is straightforward to verify that this sequence of predictions is perfectly calibrated for each outcome individually, i.e., $\Cal(\bp^{(i)}, \bx^{(i)}) = 0$ for each $i \in \{1, 2, 3\}$. For example, $p_{t, 1}$ equals $0$ for $T/3$ rounds (during which $x_t$ never equals $1$), $1/3$ for $T/3$ rounds (during which $x_t$ equals $1$ for one out of three epochs), and $2/3$ for $T/3$ rounds (during which $x_t$ equals $1$ for two out of three epochs). 

To show that $\MaxAgentReg(\bp, \bx)$ is large, we need to exhibit a specific utility function $u$ for a bounded multiclass agent. Our agent will have two actions $\cA = \{H, L\}$ with utilities as defined in Table \ref{table:binary_no_multiclass_util} (technically, this agent is not bounded in $[-1, 1]$, but it can be transformed into a bounded agent by normalizing payoffs without changing any of the results). In general, the utility for action $H$ (``high'') is much higher than the utility for action $L$ (``low''), with the exception of outcome 1, where $u(L, 1) > u(H, 1)$. In particular, $a(\beta) = H$, and for almost all choices of $p_t$, we have that $a(p_t) = H$, so $u(a_{\beta}, x_t) - u(a_{t}, x_t) = 0$ for all such rounds $t$. The only value of $p_t$ in Table \ref{table:binary_no_multiclass} where this is not the case is the single epoch when $p_t = (2/3, 1/3, 0)$ (and $x_t = 2$). For this prediction, we have that $a((2/3, 1/3, 0)) = L$ (since $u(L, (2/3, 1/3, 0)) = 13/3$, but $u(H, (2/3, 1/3, 0)) = 4$). In these $T/9$ rounds, we incur a regret of $u(a_\beta, x_t) - u(a_t, x_t) = u(H, 2) - u(L, 2) = 3$ per round, for a total regret $\AgentReg_{u}(\bp, \bx) = 3 \cdot (T/9) = T/3$. It follows that $\MaxAgentReg(\bp, \bx) = \Omega(T)$.

\begin{table}\label{table:binary_no_multiclass_util}
\begin{center}
\def\arraystretch{1.5}
\begin{tabular}{ c || c | c | c}
 $(a, x)$ & 1 & 2 & 3 \\ 
 \hline
 $H$ & $3$ & $6$ & $5$ \\
 \hline
 $L$ & $5$ & $3$ & $0$
\end{tabular}
\end{center}
\caption{Utility function for the agent in Theorem ~\ref{thm:binary_no_multiclass}.}
\end{table}
\end{proof}

One immediate corollary of Theorem \ref{thm:binary_no_multiclass} is that the class of \emph{separable} scoring rules -- scoring rules of the form $\ell(\bp, \bx) = \sum_{i=1}^{K} \ell_{i}(\bp^{(i)}, \bx^{(i)})$ for some binary scoring rules $\ell_i$ -- are not a valid answer to Question \ref{question:multiclass_v_cal}.

\begin{corollary}
There exists a sequence of $T$ multiclass events $\bx$ and predictions $\bp$ such that $\Reg_{\ell}(\bp, \bx) \leq 0$ for all bounded separable scoring rules $\ell$, but $\MaxAgentReg(\bp, \bx) = \Omega(T)$.
\end{corollary}
\begin{proof}
We use the same example as in Theorem \ref{thm:binary_no_multiclass}. Note that if $\ell(\bp, \bx) = \sum_{i=1}^{K} \ell_{i}(\bp^{(i)}, \bx^{(i)})$, then $\Reg_{\ell}(\bp, \bx) = \sum_{i=1}^{K} \Reg_{\ell^{(i)}}(\bp^{(i)}, \bx^{(i)})$. But since $\Cal(\bp^{(i)}, \bx^{(i)}) = 0$ for each $i$ in the example of Theorem \ref{thm:binary_no_multiclass}, we must have $\Reg_{\ell^{(i)}}(\bp^{(i)}, \bx^{(i)}) \leq 0$ and therefore $\Reg_{\ell}(\bp, \bx) \leq 0$. On the other hand, $\MaxAgentReg(\bp, \bx) = \Omega(T)$ (as shown in Theorem \ref{thm:binary_no_multiclass}).
\end{proof}

\begin{remark}
Again, it is interesting to compare this to the results of \cite{hartline2020optimization}. In contrast to the above result, in their paper, the authors show that optimizing over separable scoring rules \emph{does} result in an $O(K)$-worst-case approximation to optimizing over all bounded scoring rules. This apparent contradiction is resolved by examining the difference between our two settings; in \cite{hartline2020optimization}, the authors study the problem of finding a scoring rule $\ell(p)$ that optimizes the value $\int_{\Delta_K}f(p)\ell(p)dp$ for a fixed \emph{non-negative} function $f(p)$ (in fact, they take $f$ to be the pdf of a distribution). It is not possible to write $\Reg_{\ell}(\bp, \bx)$ in this way (doing so requires letting $f$ take on negative values, and also requires incorporating a linear functional of the gradient $\nabla\ell$). 
\end{remark}

\subsubsection{Finding a small generating basis for multiclass scoring rules}
\label{sec:generating_scoring_rules}

For binary classification, the class of V-shaped scoring rules has the property that the maximum regret of a V-shaped scoring rule approximates (to within a constant factor) the maximum regret of any bounded scoring rule. However, this class of scoring rules has an even stronger property: any bounded scoring rule can be written \emph{exactly} as a positive linear combination of V-shaped scoring rules (and possibly an extraneous linear function). That is, for any scoring rule $\ell \in \cL$, we can write $\ell$ in the form (up to equality on a measure zero subset) $\ell = (C_0 + C_1p) + \int_{0}^{1}\mu(v)\ell_{v}dv$ for some constants $C_0$ and $C_1$ and some measure $\mu$ over $[0, 1]$ with the property that $\int_{0}^{1}\mu(v) \leq 2$. In particular, for a fixed sequence of events $\bx$ and $\bp$, we can recover the exact value of the regret $\Reg_{\ell}(\bp, \bx)$ from the regrets $\VReg_{v}(\bp, \bx)$ of the V-shaped scoring rules (specifically, $\Reg_{\ell}(\bp, \bx) = \int_{0}^{1}\mu(v)\VReg_{v}(\bp, \bx)dv$). 

Can we hope for a similarly tight characterization of all bounded scoring rules in the multiclass setting? We show the answer is, in general, no -- when $K \geq 4$, there is no (smoothly parameterized) finite-dimensional family of functions $\cL'$ with this property.

\begin{theorem} \label{thm:extremal}

Fix a $K \geq 4$ and a finite-dimensional space of parameters, $\Theta \subseteq \Rset^N$. Let $\cL' \subseteq \cL$ be a family of bounded loss functions that are parameterized by vectors $\theta \in \Theta$ such that for any $p \in \Delta_K$, both the value of $\ell_{\theta}(p)$ and a choice of subgradient $\nabla \ell_{\theta}(p)$ are piecewise locally Lipschitz functions of $\theta$. Then there exists a loss function $\ell \in \cL$ such that it is impossible to write $\ell$ in the form

$$\ell = \int_{\theta \in \Theta} \mu(\theta)\ell_{\theta} \, d\theta$$

\noindent
for any finite measure $\mu$ over $\Theta$. 
\end{theorem}

The proof is presented in \Cref{sec:extremal}.

\subsection{An algorithm for online multiclass U-calibration}
\label{sec:multiclass_algs}

Nonetheless, despite the barriers presented in the previous section, we will demonstrate an algorithm for producing a sequence of multiclass forecasts which achieves at most $O(K\sqrt{T})$ \textit{pseudo-}U-calibration error. Here, by $O(K\sqrt{T})$ \textit{pseudo}-U-calibration, we mean that for any fixed bounded scoring rule $\ell$, the expectation (over the randomness of the forecaster) of $\Reg_{\ell}(\bp, \bx)$ is at most $O(K\sqrt{T})$. To show that this sequence of forecasts truly achieves $O(K\sqrt{T})$ expected U-calibration error, we would need to show that the expected value of $\Reg_{\ell}(\bp, \bx)$ for the \textit{worst} scoring rule $\ell$ is $O(K\sqrt{T})$; that is, we bound $\sup_{\ell \in \cL}\E[\Reg_{\ell}(\bp, \bx)]$, but to properly bound expected U-calibration error, we must bound $\E[\sup_{\ell \in \cL}\Reg_{\ell}(\bp, \bx)]$. For most practical purposes, we suspect this notion of ``pseudo'' expected U-calibration should be completely interchangeable with the actual expected U-calibration.

In contrast, note that the best algorithms we are aware of for multiclass calibration (e.g. \cite{foster1997calibrated} or \cite{blum2008regret}) only guarantee $O(T^{K/(K+1)})$ calibration error. Our algorithm below thus provides much stronger guarantees on expected agent regret than would be inherited by running one of these algorithms for calibrated forecasts. 

\begin{algorithm}[H]
\caption{$\ForecastFTPL$:}
\label{alg:forecast_ftpl}
\begin{algorithmic} 
\State For $t=1$ to $T$:
\State \quad For each $i \in [K]$, sample $n_{t, i}$ i.i.d. from the uniform distribution over $\{0, 1, 2, \dots, \lfloor \sqrt{T} \rfloor \}$.
\State \quad For each $i \in [K]$, let $\hat{X}_{t, i} = n_{t, i} + \sum_{s=1}^{t-1} \bm{1}(x_{s} = i)$.
\State \quad Output the prediction $p_{t} \in \Delta_{K}$ defined by

$$p_{t, i} = \frac{\hat{X}_{t, i}}{\sum_{j=1}^{K}\hat{X}_{t, j}}.$$
\end{algorithmic}
\end{algorithm}

At a high level, just as the algorithm $\ForecastHedge$ we presented in Section \ref{sec:algs} for the binary setting produces predictions that ``implement'' the Hedge algorithm for every individual agent, the algorithm we present here will ``implement'' a version of Follow-the-Perturbed-Leader for each individual agent. Essentially, this boils down to taking the predictions of the empirical average forecaster, but perturbing each coordinate slightly (in particular, for each outcome $i \in [K]$, we increase the count of times that outcome has historically occurred by an independent random perturbation of size roughly $O(\sqrt{T})$). We call this algorithm $\ForecastFTPL$, and a full description is presented in Algorithm \ref{alg:forecast_ftpl}.

\begin{theorem}\label{thm:forecast_ftpl}
For any sequence $\bx$ of multiclass events, $\ForecastFTPL$ produces a sequence of multiclass predictions $\bp$ which, for any bounded scoring rule $\ell \in \cL$, satisfies $\E[\Reg_{\ell}(\bp, \bx)] = O(K\sqrt{T})$. 
\end{theorem}
\begin{proof}
We begin by defining three (randomized) sequences of forecasts as a function of the sequence of outcomes $\bx$:

\begin{enumerate}
    \item $\bp^{FTPL}$ is the sequence of forecasts produced by $\ForecastFTPL$, i.e. with $p_{t, i}$ proportional to $n_{t, i} + \sum_{s=1}^{t-1} \bm{1}(x_{s} = i)$.
    \item $\bp^{BTPL}$ is the sequence of forecasts produced by the modification of $\ForecastFTPL$ where $p_{t, i}$ is proportional to $n_{t, i} + \sum_{s=1}^{t} \bm{1}(x_{s} = i)$ (i.e., ``be the perturbed leader''). 
    \item $\bp^{BTL}$ is the sequence of forecasts produced by the modification of $\ForecastFTPL$ where $p_{t, i}$ is proportional to $\sum_{s=1}^{t} \bm{1}(x_{s} = i)$ (i.e., ``be the  leader''). 
\end{enumerate}

Note that for a fixed sequence of outcomes $\bx$, $\bp^{FTPL}$ and $\bp^{BTPL}$ are random variables (that depend on the draws of noise), but $\bp^{BTL}$ is a deterministic function of $\bx$. We will actually need the $\bp^{BTL}$ forecasts for a slightly different sequence of outcomes; let $\bp^{BTL}(\bx')$ be the sequence of predictions returned by this variant for the sequence of outcomes $\bx' \in \{0, 1\}^{T'}$.

\sloppy{We first argue that there is a coupling of the random variables $\bp^{FTPL}$ and $\bp^{BTPL}$ such that $\E\left[\#\{t \mid p^{FTPL}_{t} \neq p^{BTPL}_{t}\}\right] = O(K\sqrt{T})$. To do so, let $n^{FTPL}_{t}$ be the collection of  perturbations in round $t$ for $\bp^{FTPL}$ and $n^{BTPL}_{t}$ the collection of perturbations in round $t$ for $\bp^{BTPL}$. We will couple $n^{FTPL}_{t}$ and $n^{BTPL}_{t}$ by letting (for all $t \in [T]$ and $i \in [K]$)}

\begin{equation}\label{eq:noise_coupling}
n^{BTPL}_{t, i} = \left(n^{FTPL}_{t, i} + \bm{1}(x_{t} = i)\right) \bmod \left(\lfloor\sqrt{T}\rfloor + 1\right).
\end{equation}

That is, $n^{BTPL}_{t, i}$ is equal to $n^{FTPL}_{t, i}$ if $x_{t} \neq i$ and one more than $n^{FTPL}_{t, i}$ if $x_{t} = 1$ (overflowing back to $0$ if $n^{FTPL}_{t, i}$ is already $\lfloor \sqrt{T} \rfloor$).

The coupling in \eqref{eq:noise_coupling} preserves the marginal distribution of $n^{BTPL}_{t, i}$; after coupling, all the $n^{BTPL}_{t, i}$ are still independently and distributed uniformly from $\{0, \dots, \lfloor\sqrt{T}\rfloor\}$. However, for each fixed $t \in [T]$ this coupling has the consequence that, unless $n^{FTPL}_{t, i} = \lfloor\sqrt{T}\rfloor$ for some $i \in [K]$, we will have $\hat{X}^{BTPL}_{t, i} = \hat{X}^{FTPL}_{t, i}$ for all $i \in [K]$ and hence that $p^{BTPL}_{t} = p^{FTPL}_t$. The probability that $n^{FTPL}_{t, i} = \lfloor\sqrt{T}\rfloor$ for some $i \in [K]$ is at most $K/\sqrt{T}$, and therefore $\Pr[p^{BTPL}_{t} = p^{FTPL}_t] \leq K/\sqrt{T}$ and $\E\left[\#\{t \mid p^{FTPL}_{t} \neq p^{BTPL}_{t}\}\right] = O(K\sqrt{T})$. It follows that

\begin{equation}\label{eq:ftpl_vs_btpl}
    \E[\Reg_{\ell}(\bp^{FTPL}, \bx)] \leq \E[\Reg_{\ell}(\bp^{BTPL}, \bx)] + 2K\sqrt{T}.
\end{equation}

We will next relate $\E[\Reg_{\ell}(\bp^{BTPL}, \bx)]$ to the regret of the BTL forecaster. To see this, for a fixed $\bn \in \{0, 1, 2, \dots, \lfloor \sqrt{T} \rfloor \}^{K}$, let $\bx^{(\bn)}$ be the sequence of $T + \sum_{i=1}^{K} n_i$ outcomes in $[K]$ formed by prepending $n_1$ copies of outcome $1$, $n_{2}$ copies of outcome $2$, ..., and $n_K$ copies of outcome $K$ to $\bx$. Let $|\bn| = \sum_{i=1}^{K} n_i$. Then, note that $p^{BTPL}_{t} = p^{BTL}(\bx^{(\bn_t)})_{t + |\bn|}$; that is, we can view the BTPL variant of our forecaster as running BTL with a slightly modified sequence of outcomes. We then have that (letting $\cD$ be the uniform distribution over $\{0, 1, 2, \dots, \lfloor \sqrt{T} \rfloor \}^{K}$)

\begin{eqnarray*}
    \E[\Reg_{\ell}(\bp^{BTPL}, \bx)] &=& \E\left[\sum_{t=1}^{T}\ell(p^{BTPL}_t, x_t) - \ell(\beta, x_t)\right]\\
    &=& \E_{\bn_t \sim \cD}\left[\sum_{t=1}^{T}\ell(p^{BTL}(\bx^{(\bn_{t})})_{t + |\bn_{t}|}, x_t) - \ell(\beta, x_t)\right]\\
    &=& \E_{\bn \sim \cD}\left[\sum_{t=1}^{T}\ell(p^{BTL}(\bx^{(\bn)})_{t + |\bn|}, x_t) - \ell(\beta, x_t)\right]\\
    &\leq& \E_{\bn \sim \cD}\left[\Reg_{\ell}(p^{BTL}(\bx^{(\bn)}), \bx^{(\bn)}) + \sum_{t=1}^{T}\left(\ell(\beta^{(\bn)}, x_t) - \ell(\beta, x_t)\right) + |\bn| \right] \\
    &=& \E_{\bn \sim \cD}\left[\Reg_{\ell}(p^{BTL}(\bx^{(\bn)}), \bx^{(\bn)}) + T \cdot (\ell(\beta^{(\bn)}; \beta) - \ell(\beta)) + |\bn| \right].
\end{eqnarray*}

Here we have written $\beta^{(\bn)} \in \Delta_K$ to denote the base rate forecast for the sequence of outcomes $\bx^{(\bn)}$. To conclude, note that (as a consequence of Corollary \ref{cor:multiclass_derivative_bound}), $\ell(\beta^{(\bn)}; \beta) - \ell(\beta) \leq 2||\beta^{(\bn)} - \beta||_1$. But also, note that 

\begin{eqnarray*}
|\beta^{(\bn)}_{i} - \beta_i| &=& \left|\frac{X_{t,i}}{T} - \frac{X_{t,i} + n_{i}}{T + |\bn|}\right| \\
&=& \left|\frac{X_{t,i}}{T} - \frac{X_{t,i} + n_{i}}{T} + \frac{X_{t,i} + n_{i}}{T} - \frac{X_{t,i} + n_{i}}{T + |\bn|}\right| \\
&\leq & \frac{n_i}{T} + \frac{|\bn|}{T} \cdot \beta_{i}^{(n)}.
\end{eqnarray*}

For any $\bn$ in the support of $\cD$, it follows that $||\beta^{(\bn)} - \beta||_1 = \sum_{i=1}^{K} |\beta^{(\bn)}_{i} - \beta_i| \leq 2|\bn|/\sqrt{T} \leq 2K/\sqrt{T}$. Substituting this into our earlier expression, we have that

\begin{equation}\label{eq:reg_BTPL}
    \E[\Reg_{\ell}(\bp^{BTPL}, \bx)] \leq \E_{\bn \sim \cD}\left[\Reg_{\ell}(p^{BTL}(\bx^{(\bn)}), \bx^{(\bn)})\right] + 5K\sqrt{T}.
\end{equation}

Finally, we claim that for any sequence of outcomes $\bx' \in [K]^{T}$, $\Reg_{\ell}(\bp^{BTL}(\bx'), \bx') \leq 0$. Intuitively, this follows from the fact that ``be the leader'' is a non-positive regret learning algorithm. Formally, we have that

\begin{eqnarray*}
\Reg_{\ell}(\bp^{BTL}(\bx'), \bx') &=& \sum_{t=1}^{T} \left( \ell(p_t, x'_t) - \ell(\beta(\bx'), x'_t)\right)\\
&=& \sum_{t=1}^{T} \left(\ell(p_t, x'_t) - \ell(p_{T}, x'_t)\right) \\
&=& \sum_{t=1}^{T-1} \left(\ell(p_t, x'_t) - \ell(p_{T}, x'_t)\right) \\
&\leq& \Reg_{\ell}(\bp^{BTL}(\bx'_{[1:T-1]}), \bx'_{[1:T-1]}).
\end{eqnarray*}

Here $\bx'_{[1:T-1]}$ is the truncation of $\bx'$ to all but its last entry. It then follows via induction on $T$ (combined with the base case that $BTL$ has zero regret for sequences of length one) that

\begin{equation}\label{eq:reg_BTL}
\Reg_{\ell}(\bp^{BTL}(\bx'), \bx') \leq 0.
\end{equation}

Combining \eqref{eq:ftpl_vs_btpl}, \eqref{eq:reg_BTPL}, and \eqref{eq:reg_BTL}, we find that $\E[\Reg_{\ell}(\bp^{FTPL}, \bx)] \leq 7K\sqrt{T}$. 
\end{proof}

\begin{remark}
As in the binary case, the $\sqrt{T}$ dependence on $T$ in Theorem \ref{thm:forecast_ftpl} is tight. The optimal dependence on $K$ is less clear -- the only lower bound we are aware of is the standard $\Omega(\sqrt{T\log K})$ lower bound from the learning with experts setting which extends to this problem. Is there a polynomial in $K$ lower bound for online U-calibration error?
\end{remark}

\bibliography{calib}

\appendix

\section{Comparing U-calibration to other variants of calibration}\label{sec:other-calibration}

In this appendix, we show that our notion of U-calibration is not captured by other smoothed notions of calibration. In particular, the sequence of forecasts in Example 3 of Section \ref{sec:v-calibration-examples} has low U-calibrated error, but high error with respect to all of the smoothed calibration notions mentioned in the introduction. In particular:

\begin{itemize}
    \item \textbf{Weak calibration}: In \cite{kakade2004deterministic}, a forecasting procedure is weakly calibrated if, for every Lipschitz continuous function $w: [0, 1] \rightarrow [0, 1]$,

    $$\lim_{T\rightarrow \infty}\frac{1}{T}\sum_{t=1}^{T}w(p_t)(x_t - p_t) = 0.$$

    Consider the family of sequences of forecasts in Example 3, and let $w(p) = \max(0.1 - |0.75 - p|, 0)$ (i.e., a tiny spike concentrated around $p=0.75$. The above limit does not equal $0$ for these forecasts (a constant fraction of $p_t$ will lie in the interval $[0.65, 0.85]$, but for all of those $t$, $x_t = 0$).
    \item \textbf{Smooth calibration}: In \cite{foster2018smooth}, a forecasting procedure is smooth calibrated if, for every bounded Lipschitz continuous function $\Lambda: [0, 1] \times [0, 1] \rightarrow [0, 1]$, 

    $$\lim_{T\rightarrow \infty}\frac{1}{T}\sum_{t=1}^{T}|x^{\Lambda}_t - p^{\Lambda}_t| = 0,$$

\noindent
    where

    $$x^{\Lambda}_t = \frac{\sum_{s=1}^{T}\Lambda(p_s, p_t)x_s}{\sum_{s=1}^{T}\Lambda(p_s, p_t)}$$

\noindent
    and

    $$p^{\Lambda}_t = \frac{\sum_{s=1}^{T}\Lambda(p_s, p_t)p_s}{\sum_{s=1}^{T}\Lambda(p_s, p_t)}.$$

    When $\Lambda$ only depends on its second coordinate, this is equivalent to weak calibration (so in particular, we can take $\Lambda(p_s, p_t) = \max(0.1 - |0.75 - p_t|, 0)$).
    
    \item \textbf{Continuous calibration}: \cite{foster2021forecast} define a variant of calibration called continuous calibration. Continuous calibration implies weak and smooth calibration (Appendix A.2 of \cite{foster2021forecast}), so Example 3 is not continuously calibrated.
    \item \textbf{Consistent calibration measures}: \cite{blasiok2022unifying} present a calibration metric given by the $L_1$ distance to calibration (along with two relaxations of this metric). One of their main results (Theorem 7.3 of \cite{blasiok2022unifying}) is that all of these metrics lie within a constant factor range of smooth calibration. As a result, Example 3 has high distance to calibration. 
\end{itemize}

\section{Tail Bound for Sums of Random Monotone Functions}
\label{sec:dkw-noniid}

In this section, we prove that the average of $n$ independent
random monotone functions from $\Rset$ to $[0,1]$ 
is likely to be close to its expectation, 
in the $\infty$-norm $\| F \|_{\infty} = \sup_{v \in \Rset} |F(v)| $. 
Our proof will make use of the Dvoretzky-Kiefer-Wolfowitz 
Inequality, which we restate here.

\begin{theorem}[DKW Inequality] \label{thm:dkw}
    Let $X_1,\ldots,X_n$ be i.i.d.\ random variables
    with cumulative distribution function $F$, and 
    let $\hat{F}$ denote their empirical distribution:
    \[
        \hat{F}(v) = \frac1n \, \abs{\{ i \mid X_i \le v \}} . 
    \]
    For any $\eps > 0$ we have
    \begin{equation} \label{eq:dkw}
        \Pr( \|\hat{F} - F\|_{\infty} > \eps) \leq
        2 e^{- 2 n \eps^2} .
    \end{equation}
\end{theorem}

Our tail bound for sums of non-identically distributed
monotone functions is as follows.

\newcommand{\avg}[1]{{{#1}_{\mathrm{avg}}}}
\newcommand{\Fa}{\avg{F}}
\newcommand{\given}{\, | \,}

\begin{theorem} \label{thm:dkw-noniid}
    Let $F_1,\ldots,F_n$ be independent random variables
    taking values in the set of monotone  non-decreasing functions
    from $\Rset$ to $[0,1]$, and let $\Fa = 
    \frac1n (F_1 + \cdots + F_n)$ denote their average.
    Then
    \begin{equation} \label{eq:dkw-noniid}
        \E \left[ \| \Fa - \E \Fa \|_{\infty} \right]
        \le 
        C n^{-1/2} 
    \end{equation}
    for some universal constant $C$ not depending 
    on $n$ or on the distributions of $F_1,\ldots,F_n$.
\end{theorem}
\begin{proof}
    Let $\mu$ be the distribution on $\Rset$ whose 
    cumulative distribution function is $\E \Fa$. 
    We will derive the inequality~\eqref{eq:dkw-noniid} 
    by applying the Dvoretzky-Kiefer-Wolfowitz Theorem
    applied to a collection of i.i.d.~random samples
    $Y_1,\ldots,Y_N$ from distribution $\mu$,
    where the number of random samples, $N$, is 
    itself a Poisson-distributed random variable.
    A coupling argument will allow us to relate
    the random function $\Fa$ to the empirical
    distribution of $Y_1,\ldots,Y_n$, yielding
    the desired upper bound on 
    $\E \left[ \| \Fa - \E \Fa \|_{\infty} \right].$

    In more detail, let $N$ be a Poisson-distributed
    random variable with expected value $n$. Let
    $i(1), i(2), \ldots, i(N)$ be a sequence of 
    independent samples from the uniform distribution
    on $[n] = \{1,2,\ldots,n\}$, and let $Y_1,\ldots,Y_N$
    be independent random variables such that the
    distribution of $Y_s$ given $i(s)$ has 
    cumulative distribution function $F_{i(s)}.$
    This construction has the following properties.
    \begin{enumerate}
        \item Conditional on the value of $N$, the 
        elements of the sequence $Y_1,\ldots,Y_N$ 
        are i.i.d.~random numbers each distributed
        according to $\mu$.
        \item For $i \in [n]$ let $M_i$ denote 
        the number of $s$ such that $i(s) = i.$ The
        random variables $M_1,M_2,\ldots,M_n$ are
        mutually independent Poisson random variables,
        each with expected value 1.
        \item Conditional on the value $M_i$, the 
        multiset $\mathcal{Y}_i = \{ Y_s \, \mid \, i(s) = i \}$
        is a multiset of $M_i$ i.i.d.~random numbers
        each with cumulative distribution function 
        $F_i$.
    \end{enumerate}
    Now, independently for each $i \in [n]$, let $X_i$ 
    be a random variable whose conditional distribution,
    given $\mathcal{Y}_i$, is as follows.
    If $\mathcal{Y}_i$ is non-empty, then $X_i$ equals 
    its minimum element. Otherwise, $X_i$ is randomly
    sampled from the distribution whose cumulative
    distribution function is $K_i(v) = e^{1-F_i(v)}
    - e(1 - F_i(v)) .$ (Observe that $K_i$ is monotonically
    non-decreasing,
    with $K_i(v) \to 0$ as $F_i(v) \to 0$ and 
    $K_i(v) \to 1$ as $F_i(v) \to 1$, so $K_i$ is
    indeed a cumulative distribution function.)
    Let $G_i$ and $H_i$ be the counting functions 
    of the multisets $\{X_i\}$ and $\mathcal{Y}_i$,
    respectively. In other words, 
    \begin{align*}
        G_i(v) &= \begin{cases}
            0 & \mbox{if } X_i > v \\
            1 & \mbox{if } X_i \leq v 
        \end{cases}       \\
        H_i(v) &= \abs{ \{y \in \mathcal{Y}_i \, \mid \, y \le v \}} .
    \end{align*}
    The proof will depend on the following relations.
    \begin{align}
        \label{eq:higifav}
        \forall v \; \; \E[H_i(v) \given F_1,\ldots,F_n] 
        & = \E[G_i(v) \given F_1,\ldots,F_n] = F_i(v) . \\
        \label{eq:higi}
        \forall v \; \; \E[H_i(v) \given G_i(v)] = G_i(v) .
    \end{align}
    To prove $\E[H_i(v) \given F_1,\ldots,F_n] = F_i(v),$
    first observe that $H_i(v)$ is equal to the number of 
    $s$ such that $i(s)=i$ and $Y_s \le v.$ For each $s$
    the probability that $i(s)=i$ and $Y_s \le v$, 
    given $N$ and $F_1,\ldots,F_n$, equals $\frac1n F_i(v)$.
    Hence the expected value of $H_i(v)$ given $N$ and 
    $F_1,\ldots,F_n$ is $\frac{N}{n} F_i(v).$ Since $N$
    is independent of $F_1,\ldots,F_n$ we can remove the
    conditioning on $N$ and replace it with its expected
    value, $\E[N]=n$, deriving $\E [H_i(v) \given F_1,\ldots,F_n] = F_i(v).$
    To prove $\E[G_i(v)] = F_i(v)$, we write
    \[
        \E[G_i(v)] = \Pr(G_i(v)=1) = \Pr(X_i \le v) 
    \]
    and work on proving $\Pr(X_i \le v) = F_i(v).$ 
    The event $X_i \le v$ is the union of two disjoint
    events: 
    $\mathcal{E}_1$ is the event that 
    $\mathcal{Y}_i \cap [0,v] \neq \emptyset$,
    and $\mathcal{E}_2$ is the event that 
    $\mathcal{Y}_i = \emptyset$
    and $X_i \le v.$ The number of elements
    of $\mathcal{Y}_i \cap [0,v]$ equals the 
    number of $s$ such that $i(s)=i$ and 
    $Y_s \le v$, which is a Poisson
    random variable with expected value $F_i(v).$
    Hence $\Pr(\mathcal{E}_1) = 1 - e^{-F_i(v)}.$
    By construction,
    \[
        \Pr(\mathcal{E}_2) = \Pr(\mathcal{Y}_i = \emptyset) \cdot K_i(v)
            = e^{-1} \cdot \left[ e^{1-F_i(v)} - e(1 - F_i(v)) \right]
            = e^{-F_i(v)} - 1 + F_i(v) .
    \]
    Hence, $\Pr(\mathcal{E}_1) + \Pr(\mathcal{E}_2) = F_i(v)$
    as desired. 

    We now derive Equation~\eqref{eq:higi}.
    For notational convenience, all expectation operators
    in this paragraph should be interpreted as implicit conditioning
    on $F_1,\ldots,F_n$ in addition to whatever conditioning is
    explicitly noted. First, 
    \begin{equation} \label{eq:higi0}
        \E[H_i(v) \given G_i(v)=0] = 0
    \end{equation}
    because the event $G_i(v)=0$ means 
    $X_i > v$, so either the set $\mathcal{Y}_i$
    is empty or its minimum element is greater
    than $v$, and both cases $H_i(v)=0.$
    Second, 
    \begin{align} 
        \nonumber
        \Pr(G_i(v) = 1) &= \E[G_i(v)] = \E[H_i(v)] 
        \\ & = \E[H_i(v) \given G_i(v) = 0] \cdot \Pr(G_i(v)=0)
        + \E[H_i(v) \given G_i(v)=1] \cdot \Pr(G_i(v)=1) 
        \nonumber
        \\ & =
        \E[H_i(v) \given G_i(v)=1] \cdot \Pr(G_i(v)=1) .
        \label{eq:higi1}
    \end{align}
    The first and third equations hold because $G_i(v)$ is 
    $\{0,1\}$-valued, the second is Equation~\eqref{eq:higifav},
    and the fourth holds because $\E[H_i(v) \given G_i(v)=0] = 0$.
    If $\Pr(G_i(v)=1) > 0$ then we can divide both sides 
    of Equation~\eqref{eq:higi1} by $\Pr(G_i(v)=1)$ and
    conclude that $\E[H_i(v) \given G_i(v)=1] = 1$.
    Whether or not $\Pr(G_i(v)=1) > 0$, we have shown that
    $\E[H_i(v) \given G_i(v)=x] = x$ holds for all 
    $x$ in the support of the distribution of $G_i(v),$
    so Equation~\eqref{eq:higi} is proven.

    Now, let $\avg{H} = \frac1n (H_1 + \cdots + H_n)$
    and observe that Equation~\eqref{eq:higifav}
    implies $\E[\avg{H} \given F_1,\ldots,F_n] = \Fa.$
    Using Jensen's Inequality and the convexity of 
    the $\| \cdot \|_{\infty}$ norm, we find
    \begin{equation} \label{eq:dkw-noniid.1}
        \| \Fa - \E \Fa \|_{\infty} =
        \| \E[\avg{H} - \E \Fa \given F_1,\ldots,F_n] \|_{\infty} \leq
        \E \left[ \left\| \avg{H} - \E \Fa \right\|_{\infty} \given F_1,\ldots,F_n \right] .
    \end{equation}
    Taking the expected value of both sides and using the law of iterated
    conditional expectation,
    \begin{equation} \label{eq:dkw-noniid.2}
        \E \left[ \| \Fa - \E \Fa \|_{\infty} \right]
        \leq
        \E \left[ \left\| \avg{H} - \E \Fa \right\|_{\infty} \right]
        =
        \E \left[ 
            \E \left[ \left\| \avg{H} - \E \Fa \right\|_{\infty} \given N \right]
        \right]
    \end{equation}
    Let $\hat{F}$ denote the empirical cumulative distribution 
    function of the set $\{Y_1,\ldots,Y_N\}$, or $\hat{F} \equiv 0$
    if $N=0.$ Equivalently, $\hat{F}$ is $\frac1N$ times the 
    counting function of the multiset $\{Y_1,\ldots,Y_N\}.$
    Since $\avg{H}$ is $\frac1n$ times the counting function 
    of the multiset $\{Y_1,\ldots,Y_N\}$, we have 
    $\avg{H} = \frac{N}{n} \hat{F} = \hat{F} + \left( \frac{N-n}{n} \right) \hat{F}.$ Hence,
    \begin{align*}
        \E \left[ \left\| \avg{H} - \E \Fa \right\|_{\infty} \given N \right]
        & =
        \E \left[ \left\| \hat{F} - \E \Fa +
        \left( \tfrac{N-n}{n} \right) \hat{F} \right\|_{\infty} \given N \right] \\
        & \leq
        \E \left[ \left\| \hat{F} - \E \Fa \right\|_{\infty} \given N \right] + 
        \E \left[ \left\| \left( \tfrac{N-n}{n} \right) \hat{F} \right\|_{\infty} 
        \given N \right] \\
        & \le 
        \E \left[ \left\| \hat{F} - \E \Fa \right\|_{\infty} \given N \right] + 
        \left| \tfrac{N-n}{n} \right| .
    \end{align*}
    For $N>0$ the conditional expectation on the right side can be bounded above using 
    the Dvoretzky-Kiefer-Wolfowitz Inequality.
    \begin{align*}
        \E \left[ \left\| \hat{F} - \E \Fa \right\|_{\infty} \given N \right] & = 
        \int_0^{\infty} \Pr \left( \left\| \hat{F} - \E \Fa \right\|_{\infty} > t 
        \right) \, dt \\
        & \leq 
        \int_0^{\infty} 2 e^{-2 N t^2} \, dt 
        = \int_{-\infty}^{\infty} e^{-2 N t^2} \, dt 
        = \sqrt{\frac{\pi}{2N}}.
    \end{align*}
    Removing the conditioning on $N$, we have derived
    \begin{equation} \label{eq:dkw-noniid.3}
        \E \left[ \| \Fa - \E \Fa \|_{\infty} \right]
         \le 
        \E \left[ \sqrt{\frac{\pi}{2 N}} + \left| \frac{N-n}{n} \right| \right]
    \end{equation}    
    where $N$ is a Poisson random variable with expected value $n$. 
    An application of standard tail bounds for Poisson random variables
    bounds the right side of Inequality~\eqref{eq:dkw-noniid.3}
    by $C n^{-1/2}$ for a universal constant $C$.
\end{proof}

\section{On extremal Lipschitz convex functions}
\label{sec:extremal}

\newcommand{\vc}[1]{{\bm{#1}}}
\newcommand{\iprod}[2]{{\langle #1, #2 \rangle}}
\newcommand{\vw}{{\vc{w}}}
\newcommand{\vx}{{\vc{x}}}
\newcommand{\vy}{{\vc{y}}}
\newcommand{\vz}{{\vc{z}}}
\newcommand{\ones}{{\vc{1}}}
\newcommand{\zeros}{{\vc{0}}}
\newcommand{\trans}{{\mathsf{T}}}

Let $\mathcal{D} \subseteq \Rset^d$ denote a convex
subset of $\Rset^d$.

\begin{definition} \label{defn:lin-equiv}
    We say functions $f, \, g: \mathcal{D} \to \Rset$ 
    are {\em affinely equivalent} if there are 
    non-zero scalars $a,b$ such that $af - bg$
    is an affine function from $\mathcal{D} \to \Rset$
    (i.e., the restriction to $\mathcal{D}$ of a 
    linear function plus a constant). 
\end{definition}

Note that affine equivalence is indeed 
an equivalence relation on functions: 
if $af-bg$ and $cg-dh$ are affine functions,
then $acf - bdh$ is also an affine function.

\begin{definition} \label{defn:extremal}
    We say $f : \mathcal{D} \to \Rset$ is an 
    extremal convex function on $\mathcal{D}$ 
    if:
    \begin{enumerate}
        \item $f$ is convex
        \item Whenever $\mu$ is a measure on 
        convex functions $g : \mathcal{D} \to \Rset$ 
        satisfying $f(\vx) = \int g(\vx) d\mu(g)$
        for all $\vx \in \mathcal{D}$, the set 
        of functions $g$ that are not affinely
        equivalent to $f$ has measure zero under $\mu$.
    \end{enumerate}
\end{definition}

Relations such as
\[
    f(\vx) = \left[ \tfrac13 f(\vx) + \iprod{\vw}{\vx} - 1 \right] \; + \; 
        \left[ \tfrac23 f - \iprod{\vw}{\vx} + 1 \right],
\]
which hold for any function $f$ and vectors $\vw,\vx$, 
illustrate that the conclusion ``each summand is affinely
equivalent to $f$'' in \Cref{defn:extremal} is the strongest
conclusion we can hope for. In particular, if we were to 
require that each $f_i$ equals $f$ up to scaling then no
convex function on a non-empty domain is extremal. 

The aim of this note is to prove that for any domain 
$\mathcal{D} \subseteq \Rset^3$ that contains an open 
neighborhood of $\zeros$, the set of extremal convex
functions is in some sense infinite-dimensional: it 
contains subsets that are smoothly bijectively 
parameterized by unboundedly high-dimensional 
parameter vectors.

For a finite set of vectors $W \subset \Rset^d$
let $f_W$ denote the function
\[
    f_W(\vx) = \max_{\vx \in W} \left\{ 
    \iprod{\vw}{\vx} \right\} .
\]
\begin{definition}
    \label{defn:bipyramidal}
    A finite set of vectors $W = \{ \vw_0, \vw_1, \ldots, \vw_m\}$
    in $\Rset^d$ is called {\em bipyramidal} if the 
    convex hull of $W$ includes $\zeros$ in its interior,
    and its edge set includes
    edges joining $\vw_i$ to $\vw_0$ and $\vw_m$,
    for every $i \in \{1,2,\ldots,m-1\}.$ Equivalently,
    $W$ is bipyramidal if 
    for all $i \in \{1, 2, \ldots, m-1\}$ there
    are vectors $\vx_i, \vy_i$ such that 
    \begin{align*}
        \iprod{\vw_0}{\vx_i} = 
        \iprod{\vw_i}{\vx_i} & >
        \max \left\{ \iprod{\vw_j}{\vx_i} \, \mid \,
        1 \le j \le m, \, j \neq i \right\} \\
        \iprod{\vw_m}{\vy_i} = 
        \iprod{\vw_i}{\vy_i} & >
        \max \left\{ \iprod{\vw_j}{\vy_i} \, \mid \,
        0 \le j < m, \, j \neq i \right\} .
    \end{align*}
\end{definition}

\begin{proposition} \label{prop:bipyramidal}
    If $W$ is a bipyramidal finite subset of $\Rset^d$
    and $\mathcal{D} \subseteq \Rset^d$ is a domain 
    containing an open neighborhood of $\zeros$
    then the function $f_W$ is an extremal convex
    function on $\mathcal{D}$.
\end{proposition}

To prove the proposition we will adopt the following
outline.
\begin{enumerate}
    \item \label{step:pl}
        The function $f_W$ is piecewise-linear: its
        domain $\mathcal{D}$ is partitioned into finitely
        many pieces such that the restriction of $f_W$
        to each piece is linear. For each $\vw \in W$
        the partition has a piece 
        \[
            \mathcal{D}(\vw) = \left\{
                \vx \in \mathcal{D} \mid
                \iprod{\vw}{\vx} = f_W(\vw) \right\}
        \]
        corresponding to $\vw$. Let $\Pi_W$ denote the
        partition consisting of these pieces. 
    \item \label{step:affine}
        If $\mu$ is a measure on convex functions
        $g : \mathcal{D} \to \Rset$ and 
        $f(\vx) = \int g(\vx) d \mu(g)$ for 
        every $\vx \in \mathcal{D}$, then 
        for $\mu$-almost every $g$, the 
        restriction of $g$ to each piece of 
        $\Pi_W$ is an affine function.
    \item \label{step:gradients}
        Let $G(W)$ denote the graph whose vertices
        are elements of $W$ and whose edges are pairs
        of vertices that are joined by an edge of the
        convex hull. 
        If $g : \mathcal{D} \to \Rset$ is a continuous
        function that restricts
        to an affine function on each piece of 
        the partition $\Pi_W$, then for each 
        $\vw \in W$ such that $\mathcal{D}(\vw)$
        has non-empty interior, the gradient
        of $g$ on the interior of 
        $\mathcal{D}(\vw)$ is well-defined and
        constant; denote this gradient by $g_{\vw}$.
        The next step of the proof is to show that 
        for every edge $(\vw,\vw')$ of $G(W)$, the
        vector $g_{\vw} - g_{\vw'}$ must be a scalar
        multiple of $\vw - \vw'$.
    \item \label{step:similarity}
        If $W$ is bipyramidal, and $V = (\vc{v}_0,
        \ldots,\vc{v}_m)$ is any other sequence of 
        $m+1$ vectors
        such that $\vc{v}_i - \vc{v}_j$ is a scalar
        multiple of $\vc{w}_i - \vc{w}_j$ whenever
        $(\vw_i,\vw_j)$ is an edge of $G(W)$, then 
        $V = A(W)$ for some affine function 
        $A : \Rset^d \to \Rset^d.$
    \item \label{step:conclusion}
        If $g : \mathcal{D} \to \Rset$ is a continuous
        function that restricts to an affine function 
        on each piece of the partition $\Pi_W$, then
        either $g$ itself is an affine function, or
        $g$ is affinely equivalent to $f$.
\end{enumerate}

The following lemmas encode some steps of the outline
above: \Cref{lem:affine,lem:gradients,lem:similarity}
substantiate steps~\ref{step:affine},~\ref{step:gradients},
and~\ref{step:similarity} respectively.

\begin{lemma}
    \label{lem:affine}
    If $U$ is an open subset of $\Rset$, $f : U \to \Rset$
    is an affine function,
    and $\mu$ is a measure on convex functions 
    $g : U \to \Rset$ such that $f(\vx) = \int g(\vx) \, d \mu(g)$
    for all $\vx \in U$, then for $\mu$-almost
    every $g$, the function $g$ restricted to $U$
    is an affine function. 
\end{lemma}
\begin{proof}
    Consider any three points
    $\vx,\vy,\vz \in U$ such that $\vz$ is a convex
    combination of $\vx$ and $\vy$; say, 
    $\vz = (1-\lambda) \vx + \lambda \vy$.
    For any convex function $g$ we have 
    \[ 
        (1-\lambda) g(\vx) + \lambda g(\vy) - g(\vz) \ge 0 .
    \]
    Since $f$ is affine, the left and right sides are equal when $g=f$. 
    Therefore, 
    \[
        \int (1-\lambda) g(\vx) + \lambda g(\vy) - g(\vz)
        \, d \mu(g) = 0 .
    \]
    The integrand is non-negative but the integral is zero, 
    so the integrand must be zero for $\mu$-almost every $g$. 

    Letting $\vx, \, \vy, \, \vz$ range over all triples of 
    rational points in $U$ such that $\vz$ is a convex
    combination of $\vx$ and $\vy$, we conclude that for 
    $\mu$-almost every $g$, the equation 
    $g(\vz) = (1-\lambda) g(\vx) + \lambda g(\vy)$
    holds for all such triples of rational points. 
    By continuity we may conclude that for $\mu$-almost
    every $g$
    this equation holds for all triples of points in $U$, 
    i.e.~$g$ is an affine function. 
\end{proof}

\begin{lemma}
    \label{lem:gradients}
    If $g$ is a continuous function
    that restricts to an affine function
    on each piece of the partition $\Pi_W$,
    with $g_{\vw}$ denoting the gradient
    of $g$ on the piece $\mathcal{D}(\vw)$, 
    then for every edge $(\vw,\vw')$ of $G(W)$, the
    vector $g_{\vw} - g_{\vw'}$ is a scalar
    multiple of $\vw - \vw'$.
\end{lemma}
\begin{proof}
    If $(\vw,\vw')$ is an edge of $G(W)$ then 
    there exists some $\vx \in U$ such that 
    \[
        \iprod{\vw}{\vx} = \iprod{\vw'}{\vx} > 
        \max \left\{ \iprod{\vw''}{\vx} \mid \vw'' \in W \setminus
        \{\vw,\vw'\} \right\} .
    \]
    Therefore, $\vx$ belongs to both $\mathcal{D}(\vw)$ and $\mathcal{D}(\vw')$.
    In fact, the set $\mathcal{D}(\vw) \cap \mathcal{D}(\vw')$ includes
    not only the point $\vw$, but an entire open neighborhood of $\vx$
    in the affine hyperplane 
    $H = \{ \vx' \mid \iprod{\vw - \vw'}{\vx'} = 0 \}.$
    
    If $g$ is a continuous function on $U$ that restricts to an affine
    function on each piece of $\Pi_W$, then the restriction of $g$
    to $\mathcal{D}(\vw)$ is given by some affine function 
    $\iprod{g_{\vw}}{\vx} + b$, and the restriction of $g$
    to $\mathcal{D}(\vw')$ is given by some affine function
    $\iprod{g_{\vw'}}{\vx} + b'$. Since $g$ is continuous, the 
    restrictions of these two affine functions to 
    $\mathcal{D}(\vw) \cap \mathcal{D}(\vw')$ must 
    be identical, hence 
    \[
        \forall \vx' \in \mathcal{D}(\vw) \cap \mathcal{D}(\vw') \; \;
        \iprod{g_{\vw}-g_{\vw'}}{\vx'} + b - b' = 0 .
    \]
    The function $\iprod{g_{\vw}-g_{\vw'}}{\vx'} + b - b'$ is
    an affine function that vanishes on an open subset of $H$,
    so it must vanish on all of $H$, implying that 
    $g_{\vw} - g_{\vw'}$ is a scalar multiple of the
    normal vector to $H$, namely $\vw - \vw'$.
\end{proof}

\begin{lemma}
    \label{lem:similarity}
    If $W$ is bipyramidal, and $V = (\vc{v}_0,
    \ldots,\vc{v}_m)$ is any other sequence of 
    $m+1$ vectors
    such that $\vc{v}_i - \vc{v}_j$ is a scalar
    multiple of $\vc{w}_i - \vc{w}_j$ whenever
    $i \in \{0,m\}$ and $0 < j < m$, then 
    $V = A(W)$ for some affine function 
    $A : \Rset^d \to \Rset^d.$
\end{lemma}
\begin{proof}
    Let $\vy = \vw_0 - vw_m$ and $\vz = \vc{v}_0 - \vc{v}_m$.
    For $0 < j < m,$ the vector $\vw_j$ is not collinear with
    $\vw_0$ and $\vw_m$ (as all three of them are vertices
    of the convex hull of $W$) so $\vc{w}_0 - \vc{w}_j$
    is linearly independent from $\vc{w}_m - \vc{w}_j$.
    Let $W_j$ denote the 2-dimensional linear subspace
    of $\Rset^d$ spanned by $vc{w}_0-\vc{w}_j$ and 
    $\vc{w}_m - \vc{w}_j$. Note that $\vy \in W_j$
    for all $j$. Also, since $\vc{v}_0 - \vc{v}_j$
    and $\vc{v}_m - \vc{v}_j$ are scalar multiples
    of $vc{w}_0-\vc{w}_j$ and 
    $\vc{w}_m - \vc{w}_j$, respectively, 
    the vector $\vz$ belongs to $W_j$ as well.
    The linear subspace $W_\ast = \bigcap_{j=1}^{m-1} W_j$
    contains both $\vy$ and $\vz$. However, if $W_\ast$
    were two-dimensional then we would have 
    $W_j = W_\ast$ for all $j$, implying that all
    of the vectors $\vw_i - \vw_j$ for $i \in \{0,m\}$
    and $0 < j < m$ would belong to $W_\ast$. From this
    it follows easily that the entire set $W$ belongs 
    to the two-dimensional affine space $\vw_0 + W_\ast$,
    contradicting the assumption that the convex hull of 
    $W$ contains $\zeros$ in its interior. Consequently
    $W_\ast$ cannot be two-dimensional; it must be 
    one-dimensional and hence $\vz$ is a scalar multiple
    of $\vy$, say $\vz = r \vy.$. 
    
    Now, for each $j$, let $a_j$ and $b_j$ denote coefficients
    such that $\vc{v}_0 - \vc{v}_j = a_j (\vw_0 - \vw_j)$
    and $\vc{v}_m - \vc{v}_j = b_j (\vw_m - \vw_j)$. 
    We have 
    \[
        a_j (\vw_0 - \vw_j) - b_j ( \vw_m - \vw_j) 
        = (\vc{v}_0 - \vc{v}_j) - (\vc{v}_m - \vc{v}_j) 
        = \vz = r \vy 
        = r (\vw_0 - \vw_j) - r (\vw_m - \vw_j) .
    \]
    Since $\vw_0-\vw_j$ and $\vw_m - \vw_j$ are linearly
    independent, it follows that $a_j = b_j = r.$
    In other words, 
    for all $j$, $\vc{v}_j =  
    A(\vw_j)$ where $A$ is the
    affine function $A(\vw) = 
    r \vw + (\vc{v}_0 - r \vw_0)$. 
\end{proof}

We now complete the proof of \Cref{prop:bipyramidal}.

\begin{proof}
    Suppose $W$ is a bipyramidal finite subset of $\Rset^d$ and
    $\mathcal{D} \subseteq \Rset^d$ is a domain containing an
    open neighborhood of $\zeros$. 
    For $\vw \in W$ let $\mathcal{D}(\vw)
    = \{\vx \in \mathcal{D} \mid \iprod{\vw}{\vx} = f_W(\vx) \} .$
    The sets $\mathcal{D}(\vw)$ partition $\mathcal{D}$ into polyhedral 
    cells, and the restriction of $f$ to each of these cells is 
    an affine (in fact, linear) function.
    
    By \Cref{lem:affine}, for $\mu$-almost every $g$
    the restriction of $g$ to each $\mathcal{D}(\vw)$
    is an affine function $g(\vx) = \iprod{g_{\vw}}{\vx} + b_{\vw}.$
    By \Cref{lem:gradients,lem:similarity}, the vectors
    $g_{\vw}$ as $\vw$ ranges over $W$ satisfy 
    $g_{\vw} = r (\vw - \vw_0) + \vc{v}_0$ for some scalar $r$ 
    and vector $\vc{v}_0$. Since $g$ is continuous, if the set
    $\mathcal{D}(\vw_0) \cap \mathcal{D}(\vw)$ is non-empty,
    then every $\vx \in \mathcal{D}(\vw_0) \cap \mathcal{D}(\vw)$
    must satisfy 
    \begin{align*}
        \iprod{g_{\vw}}{\vx} + b_{\vw} & =
        \iprod{g_{\vw_0}}{\vx} + b_{\vw_0}  \\
        r \iprod{\vw - \vw_0}{\vx} + \iprod{\vc{v}_0,\vx} + b_{\vw} 
        & = 
        \iprod{\vc{v}_0,\vx} + b_{\vw_0} .
    \end{align*}
    Since $\iprod{\vw - \vw_0}{\vx} = 0$ for every $\vx \in
    \mathcal{D}(\vw) \cap \mathcal{D}(\vw_0),$ it follows that
    $b_{\vw} = b_{\vw_0}$ for all $\vw$ such that 
    $\mathcal{D}(\vw) \cap \mathcal{D}(\vw_0)$ is
    non-empty. Since $W$ is bipyramidal, this includes
    every $\vw \in W$ except possibly $\vw_m$. A similar
    argument using continuity of $g$ at the points of
    $\mathcal{D}(\vw) \cap \mathcal{D}(\vw_m)$ establishes
    that $b_{\vw_m} = b_{\vw}$ for all $\vw \in W \setminus \{\vw_m\}$
    as well. 

    Summing up, we have shown that for $\mu$-almost every 
    $g$, there is a constant $b_g$ such that for all $\vw \in W$
    and all $\vx \in D(\vw),$ we have
    \[
        g(\vx) = r \iprod{\vw - \vw_0}{\vx} + \iprod{\vc{v}_0,\vx} + b_g 
          = r f_W(\vx) + \iprod{\vc{v}_0 - r \vw_0}{\vx} + b_g .
    \]
    It follows that $g(\vx) - r f_W(\vx)$ is the affine function
    $\vx \mapsto \iprod{\vc{v}_0 - r \vw_0}{\vx} + b_g,$ i.e.~$g$
    and $f_W$ are affinely equivalent, as claimed.
\end{proof}

Finally, we show how \Cref{prop:bipyramidal} implies
\Cref{thm:extremal}.

\begin{proof}
For $K \geq 4$, let $d=K-1$ and let 
$\mathcal{D} \subset \Rset^d$ denote 
the image of $\Delta_K \subset \Rset^K$ 
under an affine function that maps an
interior point of $\Delta_K$ to $\zeros$
and restricts to a one-to-one function on
$\Delta_K$. Since there is an affine 
bijection between $\Delta_K$ and $\mathcal{D}$,
the extremal convex functions on $\Delta_K$
and on $\mathcal{D}$ are in one-to-one
correspondence.

let $\mathcal{B}_N$ denote the set of $d \times N$ matrices
whose $N$ columns form a bipyramidal set in $\Rset^d$. 
Then $\mathcal{B}_N$ is an open subset of $\Rset^{d \times N}$. 
This follows directly from the observation that the
inequalities occurring in the definition of bipyramidal 
sets are strict inequalities, so bipyramidality is an
open condition.

The set $\mathcal{B}_N$ is non-empty when $N > d+1$. 
This follows because for any set $V$ of $N-2$ unit vectors in 
$\R^{d-1}$ whose convex hull contains $\zeros$ in its 
interior, the set $\{\vc{e}_1, -\vc{e}_1\} \cup \{ (0,\vc{v}) \mid 
\vc{v} \in V \}$ is an $N$-element bipyramidal set in 
$\R^d.$ Let $A_0$ denote a matrix in $\mathcal{B}_N$
whose $N$ columns are the elements of the set just 
described, and let $U$ be the
ball of radius $\delta$ (in Frobenius norm) around
$A_0$, where $\delta > 0$ is small enough that
$U \subset \mathcal{B}_N$.

Consider any $A \in U$, let $W$ be the (bipyramidal) 
set of  $N$ vectors  that form the columns of $A$,
and consider the extremal convex function $f_W$.
If the radius $\delta$ is small enough, the 
set $W$ can be reconstructed by evaluating the
subgradients $\nabla f_W(\delta \vc{a}_1), \,
\nabla f_W(\delta \vc{a}_2), \, \ldots,
\nabla f_W(\delta \vc{a}_N)$ where
$\vc{a}_1,\ldots,\vc{a}_N$ are the $N$ 
columns of the matrix $A_0$. This is 
because $\nabla f_W(\vz)$, 
for any vector $\vz$, is equal 
to the column  of $A$ that has maximum
inner product with $\vz$. The columns of
$A_0$ are unit vectors, so for each column
$\vc{a}_i$, the unique column of $A_0$ that has 
maximum inner product with $\vc{a}_i$ is
$\vc{a}_i$ itself. Since $A$ is $\delta$-close
to $A_0$ in Frobenius norm, if $\delta$ is
small enough then the unique column of $A$ that 
has maximum inner product with $\vc{a}_i$ is
the $i^{\mathrm{th}}$ column.

If $\cL'$ is a family of bounded loss functions
on $\Delta^K$ parameterized by $\theta \in \Theta \subseteq \Rset^N$,
then let $\cL''$ be the family of bounded loss 
functions on $\mathcal{D}$ obtained by precomposing
the functions in $\cL'$ with an affine 
bijection from $\mathcal{D}$ to $\Delta_K$. 
We will abuse notation and use $\ell_{\theta}$
to denote the element of $\cL''$ obtained by precomposing 
the loss function $\ell_{\theta} : \Delta_K \to [-1,1]$
with the affine bijection, so that the domain of
$\ell_{\theta}$ becomes $\mathcal{D}$ rather than
$\Delta_K$. For $i=1,2,\ldots,N$, the subgradients 
$\nabla \ell_{\theta}(\vc{a}_i)$ are piecewise 
locally Lipschitz functions of $\theta$. Since 
piecewise locally Lipschitz functions cannot
increase Hausdorff dimension, the Hausdorff dimension
of the set of $d \times N$ matrices 
formed by assembling these $N$ subgradient vectors
as $\theta$ varies over $\Theta$ is a 
$N$-Hausdorff-dimensional set of $d \times N$
matrices. The set $U \subset \mathcal{B}_N$ has
Hausdorff dimension $dN$, so there exist matrices
$A \in U$ whose columns are not the subgradients
$\{ \nabla \ell_{\theta}(\vc{a}_i) \, : \, i = 1,2,\ldots,N \}$
for any $\theta \in \Theta$. Letting $W$ be the set of 
columns of any such $A$, the loss function $f_W$ 
(treated as a function with domain $\Delta_K$) 
cannot be written in the form 
$\ell = \int_{\theta \in \Theta} \mu(\theta)\ell_{\theta} \, d\theta$
because it is an extremal convex function,
and none of the functions $\ell_{\theta}$ are
affinely equivalent to it. 
\end{proof}

\section{Omitted Proofs}
\label{app:omitted}

\subsection{Proof of Lemma~\ref{lem:sr_conv}}

\begin{proof}[Proof of Lemma~\ref{lem:sr_conv}]
The incentive compatibility constraint on scoring rules can be summarized as saying that $\sr(p) \leq \sr(p';p)$ for any $p' \neq p$. We can therefore write (for all $p \in [0,1]$)

\begin{equation}\label{eq:sr_conv2}
\sr(p) = \min_{p' \in [0, 1]} \sr(p';p).
\end{equation}

Since for each fixed $p'$, $\sr(p';p)$ is a linear function in $p$, \eqref{eq:sr_conv2} shows that $\sr(p)$ is a minimum of a set of linear functions and is therefore concave.

Conversely, let $f: [0, 1] \rightarrow [0, 1]$ be a concave function. Let $f'(p)$ be a subgradient of $f$; i.e., $f'$ satisfies $f(p) \leq f(q) + (p-q)f'(q)$ for all $p, q \in [0, 1]$. Then, if we define the scoring rule $\sr$ via

\begin{equation}\label{eq:scoring_rule_via_derivative_app}
\sr(p, 0) = f(p) - pf'(p) \qquad \qquad 
\sr(p, 1) = f(p) + (1-p)f'(p),
\end{equation}

\noindent
we claim that $\sr$ is a proper scoring rule with the property that $\sr(p) = f(p)$. First, note that

$$\sr(p) = (1-p)(f(p) - pf'(p)) + p(f(p) + (1-p)f'(p)) = f(p),$$

\noindent
so $\sr(p) = f(p)$ as desired. Secondly, note that 

$$\sr(q; p) = (1-p)(f(q) - qf'(q)) + p(f(q) + (1-q)f'(q)) = f(q) + (p-q)f'(q).$$

Therefore, the concavity condition on $f$ immediately implies that $\sr(p) \leq \sr(q; p)$ for all $p, q \in [0,1]$, and therefore $\sr$ is a proper scoring rule. 
\end{proof}

\subsection{Proof of Theorem~\ref{thm:cal_suffices_for_swap}}

\begin{proof}[Proof of Theorem~\ref{thm:cal_suffices_for_swap}]
We follow the structure of the proof of Theorem \ref{thm:cal_suffices}. Let $\sr = -\wtu$ be the proper scoring rule corresponding to $u$. For each $p \in \{p_t\}$, let $\hat{p} = m_p/n_p$ (the empirical outcome on the rounds where $p$ was predicted). Finally, let $\pi: \cA \rightarrow \cA$ be the swap function maximizing $\sum_{t=1}^{T} u(\pi(a(p_t)), x_t)$. Note that the optimal choice of $\pi$ sets $\pi(a(p)) = a(\hat{p})$. We then have that

\begin{eqnarray*}
    \AgentSwapReg_{u}(\bp, \bx) &=& \sum_{t=1}^{T} u(\pi(a(p_t)), x_t) - \sum_{t=1}^{T} u(a(p_t), x_t) \\
    &=& \sum_{t=1}^{T} u(a(\hat{p_t}), x_t) - \sum_{t=1}^{T} u(a(p_t), x_t) \\
    &=& \sum_{p \in [0, 1]} \sum_{t; p_t = p}\left(\wt{u}(\hat{p}, x_t) - \wt{u}(p, x_t)\right) \\
    &=& \sum_{p \in [0, 1]} \sum_{t; p_t = p}\left(\sr(p, x_t) - \sr(\hat{p}, x_t)\right) \\
    &=& \sum_{p \in [0, 1]} \left((n_p - m_p)(\sr(p,0) - \sr(\hat{p},0)) + m_p(\sr(p,1) - \sr(\hat{p},1))\right) \\
    &=& \sum_{p \in [0, 1]} n_p\left(\sr\left(p;\hat{p}\right) - \sr\left(\hat{p}; \hat{p}\right)\right) \\
    &\leq& 4 \sum_{p \in [0, 1]} n_p \left| p - \frac{m_p}{n_p}\right| = 4 \Cal(\bp, \bx).
\end{eqnarray*}

\noindent
Here the last inequality follows from applying  \eqref{eq:sr_conv3}.

\end{proof}

\subsection{Proof of Lemma~\ref{lem:uni_to_bi}}

\begin{proof}[Proof of Lemma~\ref{lem:uni_to_bi}]
We follow the outline of the proof of Lemma \ref{lem:sr_conv}. By the incentive compatibility constraint, we have that $\sr(p) = \min_{p' \in \Delta_K} \sr(p'; p)$; since each $\sr(p'; p)$ is a linear function in $p$, this implies $\sr(p)$ is concave. 

Conversely, if $f(p)$ is a concave function with subgradient $\nabla f(p)$, then if we define $\sr(p, x) = f(p) + \langle x - p, \nabla f(p) \rangle$, note that $\sr(p) = \E_{x \sim p}[f(p) + \langle x - p, \nabla f(p) \rangle] = f(p) + \langle p - p, \nabla f(p) \rangle = f(p)$. To see that this is a valid scoring rule, note that $\sr(q;p) = \E_{x \sim p}[f(q) + \langle x-q, \nabla f(q) \rangle] = f(q) + \langle p-q, \nabla f(q) \rangle$. By the concavity of $f$, this is at least $f(p) = \sr(p)$, and therefore $\sr$ satisfies the required incentive constraints.

Finally, note that since $\ell(p; q)$ is a linear function in $q$ with the property that $\ell(p; q) \leq \ell(p)$ for all $q \in \Delta_K$ and $\ell(p;p) = \ell(p)$, $\ell(p;q)$ must be a tangent hyperplane to $\ell(p)$ at $p$ and (if $\ell$ is uniquely differentiable) must have the form $\ell(p;q) = \ell(p) + \langle q - p, \nabla\sr(p) \rangle$. Substituting the unit vectors for $q$, we obtain equation \eqref{eq:uni_to_bi}.
\end{proof}

\subsection{Proof of Theorem~\ref{thm:multiclass_cal_suffices}}

From Lemma \ref{lem:uni_to_bi}, we first establish a bound on the norm of the gradient of a multiclass scoring rule analogous to that of Corollary \ref{cor:derivative_bound}.

\begin{corollary}\label{cor:multiclass_derivative_bound}
For any $p, q, q' \in \Delta_{K}$, $\langle q - q', \nabla \sr(p) \rangle \leq ||q - q'||_1$. 
\end{corollary}
\begin{proof}
By equation \eqref{eq:uni_to_bi}, we have that $\langle q - q', \nabla \sr(p) \rangle = \sr(p;q') - \sr(p;q) = \sum_{i=1}^{K}(q'_i - q_i)\ell(p, i) \leq \sum_{i=1}^{K}|q'_i - q_i| = ||q - q'||_1$.
\end{proof}

\begin{proof}[Proof of Theorem~\ref{thm:multiclass_cal_suffices}]
Let $\sr$ be the scoring rule corresponding to this agent (so we wish to show that $\Reg_{\sr}(\bp, \bx) \leq 2\Cal(\bp, \bx)$). As in the proof of Theorem \ref{thm:cal_suffices}, we begin by showing that (for any $p, \hat{p} \in \Delta_K$)

\begin{equation}\label{eq:multiclass_cal_sr_bound}
    \sr(\hat{p}) \leq \sr(p;\hat{p}) \leq \sr(\hat{p}) + 2 ||p - \hat{p}||_1.
\end{equation}

The first inequality follows from the fact that $\sr$ is a proper scoring rule. To show the second inequality, first note that we have that $\sr(p;\hat{p}) = \sr(p) + \langle \hat{p} - p, \nabla \sr(p) \rangle$ by equation \eqref{eq:uni_to_bi}, which in turn is at most $\sr(p) + || p - \hat{p}||_1$ by Corollary \ref{cor:multiclass_derivative_bound}. Similarly, we have that $\sr(p) \leq \sr(\hat{p}) + || p - \hat{p}||_1$, since (by concavity of $\sr$) $\sr(p) \leq \sr(\hat{p}) + \langle p - \hat{p}, \nabla \sr(\hat{p}) \rangle \leq \sr(\hat{p}) + || p - \hat{p}||_1$. Combining these two inequalities, we obtain the second inequality in \eqref{eq:multiclass_cal_sr_bound}.

Now, define $n_p = \abs{\{t; p_t = p\}}$ and $\hat{p} = \frac{1}{n_p} \sum_{t; p_t = p} x_t$. Then, we have that

\begin{eqnarray*}
\Reg_{\sr}(\bp, \bx) &=& \sum_{t=1}^{T} \sr(p_t, x_t) - \sum_{t=1}^{T} \sr(\beta, x_t) \\
&=& \sum_{p \in \Delta_K} \sum_{t\,;\,p_t = p} (\sr(p, x_t) - \sr(\beta, x_t)) \\
&=& \sum_{p \in \Delta_K} n_p\left(\sr(p;\hat{p}) - \sr(\beta;\hat{p})\right) \\
&\leq & 2\sum_{p \in \Delta_{K}} n_p ||p - \hat{p}||_1 \\
&=& 2\Cal(\bp, \bx).
\end{eqnarray*}

The last inequality here follows from applying \eqref{eq:multiclass_cal_sr_bound} to both $\sr(p;\hat{p})$ and $\sr(\beta;\hat{p})$. 
\end{proof}

\end{document}